\documentclass[11pt]{article}
\pdfoutput=1
\usepackage[OT1]{fontenc}
\usepackage[usenames]{color}

\usepackage[protrusion=true,expansion=true,final,babel]{microtype}

\usepackage{fullpage}

\usepackage{smile}
\usepackage{cme-math}

\usepackage[colorlinks,linkcolor=red,anchorcolor=blue,citecolor=blue]{hyperref}

\graphicspath{{./plots/}}

\newcommand\bCS[1][]{\Cb_{ #1 }^*}
\newcommand\CS[2]{C_{#1,#2}^*}

\newcommand\bSigmaS[1][]{\bSigma_{ #1 }^*}
\newcommand\bSigmaH[1][]{\hat{\bSigma}_{ #1 }}
\newcommand\SigmaH[2]{\hat{\Sigma}_{ #1, #2 }}

\newcommand\bGammaH{\hat{\bGamma}}
\newcommand\GammaH[2]{\hat{\Gamma}_{#1,#2}}
\newcommand\gammaS[1]{\gamma_{#1}^*}
\newcommand\bGammaS{\bGamma^*}

\newcommand\bThetaS[1][]{\bTheta_{ #1}^*}

\newcommand\Gh[1]{\hat{G}_{#1}}
\newcommand\Gs[1]{G^*_{#1}}

\newcommand\lminf[1]{\lambda_{\min,F}\left(#1\right)}
\newcommand\lmin[1]{\lambda_{\min}\left(#1\right)}
\newcommand\lmax[1]{\lambda_{\max}\left(#1\right)}
\newcommand\psdc[1]{\cS^{#1 \times #1}_+}
\newcommand\trip[2]{\langle #1, #2 \rangle}

\newcommand\bZero{\mathbf{0}}
\newcommand\vect{\mathop{\text{vec}}}

\newcommand\dvect{\mathop{\text{dvec}}}

\newcommand\yah[1]{y_{#1}}
\newcommand\yabh[2]{y_{#1,#2}}
\newcommand\yTh{y_{T}}

\begin{document}
\date{October 18, 2018}
\title{\huge Efficient, Certifiably Optimal Clustering with Applications to Latent Variable Graphical Models}
\author{Carson Eisenach\thanks{Department of Operations Research and Financial Engineering, Princeton University, Princeton NJ 08544, USA; e-mail: {\tt eisenach@princeton.edu}}  \and Han Liu\thanks{Department of Electrical Engineering and Computer Science, Northwestern University, Evanston IL 60208, USA}}

\maketitle

\begin{abstract}
Motivated by the task of clustering either $d$ variables or $d$ points into $K$ groups, we investigate efficient algorithms to solve the Peng-Wei (P-W) $K$-means semi-definite programming (SDP) relaxation.  The P-W SDP has been shown in the literature to have good statistical properties in a variety of settings, but remains intractable to solve in practice. To this end we propose FORCE, a new algorithm to solve this SDP relaxation. Compared to the naive interior point method, our method reduces the computational complexity of solving the SDP from $\tilde{\cO}(d^7\log\epsilon^{-1})$ to $\tilde{\cO}(d^{6}K^{-2}\epsilon^{-1})$ arithmetic operations for an $\epsilon$-optimal solution. Our method combines a primal first-order method with a dual optimality certificate search, which when successful, allows for early termination of the primal method. We show for certain {\it variable clustering problems} that, with high probability, FORCE is guaranteed to find the optimal solution to the SDP relaxation and provide a certificate of exact optimality. As verified by our numerical experiments, this allows FORCE to solve the P-W SDP with dimensions in the hundreds in only tens of seconds. For a variation of the P-W SDP where $K$ is not known a priori a slight modification of FORCE reduces the computational complexity of solving this problem as well: from $\tilde{\cO}(d^7\log\epsilon^{-1})$ using a standard SDP solver to $\tilde{\cO}(d^{4}\epsilon^{-1})$.
\end{abstract}

\section{Introduction}
\label{sec:introduction}
Clustering a set of objects optimally according to some similarity measure is a central task of statistics and machine learning. These problems arise everywhere from the analysis of medical imaging data to search result groupings on Google. Such tasks can be broadly categorized as either: {\it data clustering}, where we partition of $d$ points in $\RR^p$ into $K$ clusters, or {\it variable clustering}, where we consider $n$ samples of a random variable $\bX \in \RR^d$ and group the variables into $K$ groups of size at least $m$. In many actual use cases the purpose of clustering is to recover some underlying {\it ground truth}, a partition $\cG^* = \{\Gs{1},\dots,\Gs{K}\}$; the optimization objective and similarity measure are chosen such that the optimal partitioning corresponds to the ground-truth.

For data clustering, one classical formulation is $K$-means:
\begin{equation}
\label{eqn:kmeans_nphard}
\argmin_{G} \sum_{s=1}^K \sum_{i \in G_s} ||x_i - \mu_{s}||_2^2, \text{ subject to } \mu_s = \frac{1}{|G_s|} \sum_{i \in G_s} x_i,
\end{equation}
This formulation, roughly speaking, can also be applied to variable clustering by treating $\Cov(\bX)$ as a measure of ``distances'' between $d$ points \citep{Bunea2016}. Because \eqref{eqn:kmeans_nphard}, and combinatorial optimization in general, is NP-hard \citep{Dasgupta2008,Mahajan2012}, fast algorithms that have been proposed to solve clustering problems are not guaranteed to produce an optimal solution to the original problem \citep{Lloyd1982,Defays1977,Kumar2010,Arthur2007,Peng2007}.

This becomes a major issue in certain scenarios, like post-selection inference, where first a statistical model is selected, e.g. through variable clustering, and then an inferential procedure is applied. Nearly optimal clusterings are insufficient for this purpose because incorrect model selection will invalidate the results of subsequent inferences; for such applications recovery of the optimal clustering is required. Applications where variable clustering and statistical inference questions arise include the analysis of stock pricing, fMRI, and gene expression data.

One particularly interesting class of algorithms leverage a {\it convex relaxation} to find an approximate solution, followed by a rounding step \citep{Vazirani2001}. Though this may not always give an optimal solution to the original problem, significant progress has been made on understanding when such relaxations are {\it tight} -- that the optimal solution to the relaxed and original problems coincide \citep{Awasthi2014,Peng2007,Bunea2016,Iguchi2015}. Motivated by recent developments in cluster based graphical models, in particular the $G$-Latent model (see Section~\ref{sec:preliminaries}) where each cluster of variables corresponds to a latent generator \citep{Bunea2016,Bunea2018,Bunea2017}, we study efficient algorithms for exact cluster recovery.

\citet{Bunea2016} show that the Peng-Wei (P-W) SDP relaxation (see Section~\ref{sec:preliminaries}) of \eqref{eqn:kmeans_nphard} is tight with high probability for $G$-Latent models and introduce a procedure to recover $\cG^*$ based on solving this SDP. Similarly recent work \citep{Awasthi2014,Bandeira2015} has studied when convex relaxations are tight in the data clustering setting. In this setting it is again the P-W SDP which has the strongest statistical guarantees \cite{Ames2014,Awasthi2014,Iguchi2015a,Iguchi2015}.

Despite the attractive theoretical properties of the P-W SDP for a variety of clustering problems, efficiently solving it in practice remains a significant challenge: standard SDP solvers have worst-case $\tilde{\cO}(d^7\log\varepsilon^{-1})$ running time due to a large number of constraints. In this paper we introduce FORCE ({\bf F}irst-{\bf OR}der {\bf CE}rtifiably Optimal Clustering), an algorithm to solve the P-W SDP. The difficulty in solving NP-hard problems, such as $K$-means, derives from the integer structure of their solutions. The underlying insight is that for clustering problems, when we expect the convex relaxation to be tight, the integer structure of the optimal solution can actually be leveraged to {\it help solve} the clustering problem. The FORCE algorithm consists of two components: a first-order method to solve the P-W SDP and a dual solution construction used to certify the optimality of a primal solution. The idea is that if we have an algorithm to quickly construct a dual solution at $G^*$ and an interior point method to solve an SDP relaxation $\cP$, then while solving $\cP$ we can periodically ``round'' the current iterate and search for a matching dual solution. If the primal and dual objective values match, the algorithm can terminate early.

We summarize our main contributions below:
\begin{enumerate}
\item {\bf FORCE Primal Step and Convergence Analysis:} A first-order algorithm for the P-W SDP based on a variant of Renegar's Smoothed Scheme (RSS) \citep{Renegar2014}. By converting the SDP to an eigenvalue maximization problem, we obtain a substantially improved convergence rate because we can reduce the {\it effective dimension} of the problem from $\cO(d^2)$ to $\cO(d)$. This allows us to reduce the number of arithmetic operations required to approximately solve the P-W SDP from $\tilde{\cO}(d^7\log\varepsilon^{-1})$ to $\tilde{\cO}(d^{6}K^{-2}\epsilon^{-1})$.\footnote{Note that $\epsilon$ corresponds to a type of relative additive error where as $\varepsilon$ corresponds to additive error.}
\item {\bf Dual Certificate and Probabilistic Guarantees for Variable Clustering:} We introduce a novel dual certificate for the P-W SDP that is tailored to variable clustering and easy to compute. 
We show that for clustering in $G$-Latent models, this certificate is guaranteed to exist with high probability (w.h.p.) at a nearly the minimax optimal cluster separation rate required for recovery of $G^*$.
\item {\bf Extensions to Unknown $K$:} We extend FORCE to a P-W SDP variant recently considered for variable clustering when $K$ is not known \citep{Bunea2016}. Theoretical guarantees translate almost 1-to-1 from the case when $K$ is fixed, except now the FORCE primal step requires $\tilde{\cO}(d^{4}\epsilon^{-1})$ arithmetic operations to obtain an $\epsilon$-approximate solution.
\end{enumerate}

\begin{remark}
We make no claims as to the statistical properties of the dual certificate for other generative models for the clustering data -- e.g. for the stochastic block model or stochastic ball model. In general, the design of an appropriate dual certificate is closely linked to the data generating distribution. In any case the primal step is still applicable -- to the best of our knowledge our proposed method is the most efficient algorithm to date for solving the P-W SDP -- and in practice the dual certificate may be useful even if it is not guaranteed to exist w.h.p., but this is beyond the scope of our work.
\end{remark}

\begin{remark}
Our theoretical analysis of the statistical properties of the proposed dual certificate also provides an alternative proof of the tightness of the P-W SDP for variable clustering in $G$-Latent models, at nearly the same cluster separation rate as in the literature \citep{Bunea2016}. This proof differs from \citet{Bunea2016} in that it is more constructive in nature since it analyzes the properties of an explicit dual solution construction. It also shows that instances are perfectly recoverable using and can be proven optimal for the P-W SDP at nearly the same cluster separation rate. 
\end{remark}

\paragraph{Notation.}
Denote either a clustering of data points or a partition of variables by $G = \{G_1,\dots,G_K\}$ where $G_i$ is a single cluster or variable group. Hats, i.e. $\hat G$, always indicate quantities estimated from data and stars, i.e. $G^*$, always denote ground truths. For a $n \times n$ matrix $\Mb$, $||\Mb||_2$ denotes the largest eigenvalue of $\Mb$ and $||\Mb||_{\infty}$ is the matrix $\ell_{\infty}$ norm. $||\Mb||_{\max} = \max_{i,j}|M_{i,j}|$ and $||\Mb||_{\min} = \min_{i,j}M_{i,j}$. Let $S$ and $S'$ be subsets of $[n]$. Then $\Mb_{S,S'}$ refers to the sub-matrix of $M$ with entries whose row index is in $S$ and column index is in $S'$. The notation $\tilde\cO$ is used to suppress poly-log factors of the dimension $d$. The function $\lambda(\Mb)$ maps a matrix  $\Mb$  to the set of its eigenvalues. Similarly $\lambda_{\min}(\Mb)$ and $\lambda_{\max}(\Mb)$ map $\Mb$ to its minimum and maximum eigenvalue, respectively. We define $\dvect(\Mb):=\diag(\vect(\Mb))$, mapping a matrix $\Mb$ to a diagonal matrix with the vectorized matrix $\Mb$ on the main diagonal.

\section{Preliminaries}

\subsection{Background}
\label{sec:preliminaries}
\paragraph{Peng-Wei SDP.}
The Peng-Wei SDP \citep{Peng2007} is defined as
\begin{equation}
\underset{\Ub}{\text{maximize}} \trip{- \Db}{\Ub}  ~~\text{s.t.}~~ \Ub \in \cC := \{ \Ub : \Ub \geq 0; \Ub \bone = \bone; \tr(\Ub) = K; \Ub \succeq 0\}.
\label{eqn:kmeans_sdp}
\end{equation}
For the data clustering problem, $\Db$ is defined by $D_{i,j} = ||x_i - x_j||_2^2$. A solution is called ``integer'' if $U_{ij} = \frac{1}{|G_a|}$ if $i,j \in G_a$ and 0 otherwise, and it is said to correspond to the partition $G$. This is also called the ``partnership matrix'' of the clustering solution $G$, which we denote by $B(G)$. It can be shown that the dual SDP to \eqref{eqn:kmeans_sdp} is
\begin{equation}
\begin{aligned}
& \underset{y_{a,b},y_a,y_T}{\text{minimize}} & & 2\sum_{a=1}^d y_a + Ky_T \\
& \text{subject to} & & \sum_{a=1}^d y_a \Rb_a + y_T \Ib \succeq - \Db + \sum_{a \leq b} y_{a,b} \Ib_{a,b} \\
&&& y_{a,b} \geq 0 \text{ for all } a \leq b,
\end{aligned}
\label{eqn:kmeans_dual}
\end{equation}
where the matrices $\Ib_{a,b}$ and $\Rb_a$ are defined by $\Ib_{a b} = \frac{1}{2}\left(\be_a \be_b^T + \be_b \be_a^T \right)$ for all $a < b$, $\Ib_{a a} = \frac{1}{2} \be_a \be_a^T$, and $\Rb_{a} = \bone \be_a^T + \be_a^T \bone$.

\paragraph{Variable Clustering in G-Latent Models.}
The $G$-Latent model assumes the observed variables $\bX =(X_1, \ldots, X_d) \in \RR^d$ can be partitioned into $K$ unknown clusters $G^* = \{\Gs{1}, \ldots, \Gs{K}\}$ such that variables in the same cluster share similar behavior. We denote $m := \min_{i}|\Gs{i}|$ and assume that $m\geq 3$. Further we also assume there exists a latent mean-zero random vector $\bZ \in \RR^K$ with covariance matrix $\Cov(\bZ) = \bCS$, such that $\bX=\Ab \bZ+\bE$, for a zero mean error vector $\bE$ with independent entries. The $d \times K$ assignment matrix $\Ab$ is defined as  $A_{jk}=\II\{j\in G^*_{k}\}$. We denote $\Cov(\bE)=\bGammaS$, a diagonal matrix with entries $\Gamma^*_{jj} = \gammaS{j}$ for any $1\leq j\leq d$.  We also assume that the noise $\bE$ is independent of $\bZ$. We assume that $\bZ\sim \cN(0,\bCS)$ and $\bE \sim \cN(0,\bGammaS)$, which implies $\bX \sim \cN(0,\bSigmaS)$ with $\bSigmaS=\Ab\bCS\Ab^T+\bGammaS$. To be able to recover clusters, the latent variables cannot be too highly correlated, and we can define a distance between components of $\bZ$ as
\[
\Delta(\bCS)=: \min_{j < k} \EE(Z_j - Z_k)^2  > 0.
\]

To recover the true group partition $G^*$, \citet{Bunea2016} propose using \eqref{eqn:kmeans_sdp} with $\Db = \bGammaH - \bSigmaH$, a penalized covariance matrix estimator (we refer to this as the PECOK estimator). Because a priori the group structure is unknown, an estimator $\hat\bGamma$ of $\bGammaS$ is somewhat involved so we omit the details here. For our purposes, we are only concerned with its rate of convergence in the max-norm. \citet{Bunea2016} show that if $\Xb_i,\dots,\Xb_n$ are generated from a G-Latent Model, there exist constants $p_0-p_2$ such that if  $\log d \leq p_0 n$, then with probability at least $1-p_2/d^3$,
\begin{equation}
\label{eqn:conc_pecok_gamma}
||\hat\bGamma - \bGamma^*||_{\infty} \leq p_1||\bGamma^*||_{\infty}\sqrt{\log d/n} =: \delta_{n,d}.
\end{equation}
Furthermore, if 
\[
\Delta(\bC^*) \gtrsim \|\bGamma^*\|_{\infty}\left(\sqrt{\frac{\log d}{nm}} + \sqrt{\frac{\log d}{nm^2}} +\frac{d}{nm} + \frac{\log d}{n}\right) + \frac{\delta_{n,d}}{m},
\]
then with probability at least $1-p_3/d$ the optimizer to \eqref{eqn:kmeans_sdp} is $\Xb^* = B(G^*)$ for some constants $p_0$, $p_3$. This bound on $\Delta(\bCS)$ is shown to be minimax optimal.

\subsection{Related Work}

\paragraph{Solving the SDP.}
An obvious approach is to simply solve the P-W SDP relaxation using the standard second-order convex optimization methods (see \citet{Boyd2004} for some examples). One well known approach to quickly solving certain SDPs is the matrix multiplicative weights (MMW) algorithm \citep{Arora2005}. For the P-W SDP, the MMW algorithm requires $\tilde{\cO}(K^2d^2\alpha^{-2}\epsilon^{-2})$ arithmetic operations to find an $\epsilon$-optimal\footnote{Here $\epsilon$ is a multiplicative error} solution and where $\alpha$ is related to a lower bound on the optimal value of a rescaled version of the SDP. Typically, we have $\alpha = \cO(d^{-1})$ giving a computational complexity of $\tilde{\cO}(K^2d^4\epsilon^{-2})$.\footnote{We did implement a MMW algorithm for P-W SDP, but found it unable to converge in practice; we suspect this is due to the presence $d^2$ equality constraints since at each iteration of MMW these are not satisfied, but we did not investigate this further.}

Another possibility is to solve the SDP using the Alternating Direction Method of Multipliers (ADMM) \citep{Boyd2011}. Recent work \citep{Ames2014} 
takes this approach for a related SDP relaxation applied to the {\it bi-clustering problem}, but the focus there is on the statistical proeprties of the SDP relaxation not on deriving an algorithm with convergence guarantees. Like our approach, ADMM requires $\cO(d^3)$ arithmetic operations per update, but there is no gaurantee on its convergence rate. Instead, in this paper we convert the SDP into an equivalent eigenvalue maximization problem using a technique due to \citet{Renegar2014}, which allows us to achieve better worst-case runtime bounds than existing methods. This is described in more detail in the next section.

\paragraph{Optimality Certificates for Data Clustering.}
In proving the tightness of \eqref{eqn:kmeans_sdp} it is standard to derive an empirically testable condition on an instance of the clustering problem \citep{Awasthi2014,Iguchi2015a,Iguchi2015}.  To do this, recent work on convex relaxations of $K$-means for data clustering takes a {\it dual optimality certificate} approach \citep{Awasthi2014,Iguchi2015a,Iguchi2015}. In general the dual optimality certificate approach is: (a) find an appropriate convex relaxation (denoted $\cP$) and its dual (denoted $\cD$) of \eqref{eqn:kmeans_nphard}, (b) given a candidate solution to $\cP$ construct a solution to $\cD$ with matching objective value, (c) derive a deterministic condition that can be checked on an instance of $\cP$ and proposed solution to $\cP$ that is sufficient for the construction in (b) to exist. The deterministic condition found in step (c) can then be analyzed to find the necessary assumptions on the data generating distribution to give the following guarantee: {\it with high probability a random instance of $\cP$ will satisfy the condition at the optimal solution $G^*$ to $\cP$.} To use the condition from step (c), all that remains is a way to ``quickly'' find optimal solutions to $\cP$ and then test the condition at the proposed optimal solution.

The dual solutions used in \citet{Awasthi2014,Iguchi2015} differ from each other mainly in their choice of assignment to $y_{a,b}$ (likewise for our proposed certificate). The choice of $y_{a,b}$ in turn determines what testable condition one can derive and then leverage to prove tightness results and certify optimal clusterings. Unfortunately, \citet{Iguchi2015} only offer a fast algorithm for the $K=2$ case, and their method cannot be directly applied to variable clustering since it operates directly on the data points to be clustered, not merely the matrix $\Db$. These certificates (and ours) benefit from Lemma \ref{lem:dual_sdp_sol} characterizing solutions to \eqref{eqn:kmeans_dual}.

\begin{lemma}[Theorem 4 \citep{Iguchi2015a}]
\label{lem:dual_sdp_sol}
The following are equivalent: (a) $\Bb^*$ is an optimal solution to \eqref{eqn:kmeans_sdp}, (b) every solution to \eqref{eqn:kmeans_dual} satisfies $y_{a,b} = 0$  for $a,b \in \Gs{i}$ and $\Qb_{\Gs{i},\Gs{i}}\bone = 0$  for all $i$, and (c) every solution to \eqref{eqn:kmeans_dual} satisfies $\yb_{\Gs{i}} = \Lb^{-1}_{\Gs{i},\Gs{i}}(-\Db_{\Gs{i},\Gs{i}}\bone - y_T\bone)$. $\Lb$ is a block-diagonal matrix determined by $G^*$, where the diagonal blocks are defined as $\Lb_{\Gs{i},\Gs{i}} = |\Gs{i}|\Ib + \bone\bone^T$ and the off-diagonal blocks are zero.
\end{lemma}

\paragraph{Other Clustering Approaches.}
Spectral clustering \citep{Kumar2010,Awasthi2012} is another approach, but these methods are tailored towards data clustering and are provably suboptimal \citep{Bunea2016} in terms of exact recovery in variable clustering. Heuristic approaches such as Lloyd's Algorithm \citep{Lloyd1982} and CLINK \citep{Defays1977} are fast, but in general do not find global optima.

\paragraph{Comparison To Stochastic Block Model.}
Variable clustering of data generated by the stochastic block model (SBM) has been heavily studied in recent years using the P-W SDP (and other related SDPs). In SBM, one wants to recover the true partition of $d$ nodes using an observed $d\times d$ adjacency matrix where each entry is modeled as an independent Bernoulli random variable. Similar recovery guarantees to those described for the $G$-Latent model exist for SBM and use similar proof techniques \citep{Abbe2016,Ames2014,Pirinen2016}. An effective algorithm for solving the P-W SDP could therefore also benefit clustering in this regime as well.

\section{The FORCE Algorithm}
\label{sec:FORCE_all}
In this section we first present the primal step, followed by the dual certificate and then a convergence guarantee for the P-W SDP on any instance $\cD$.

\subsection{Primal Step}
\label{sec:FORCE_primal}
Because we consider clustering in the high-dimensional setting, a fast algorithm to solve \eqref{eqn:kmeans_sdp} is critical. While second-order methods have an appealing iteration complexity, the per iteration cost is prohibitive for \eqref{eqn:kmeans_sdp} because the cost of each iteration depends not only on the dimension $d$ {\it but also on the number of constraints} -- in \eqref{eqn:kmeans_sdp}, this is $\cO(d^2)$. First-order methods, by contrast, may have a higher iteration complexity, but a lower per-iteration cost.

\paragraph{Algorithmic Framework}
Informally, RSS \citep{Renegar2014} can be described as Nesterov's accelerated gradient method \citep{Nesterov2004} and smoothing \citep{Nesterov2005,Nesterov2007} applied to an eigenvalue maximization problem that is closely linked to the SDP of interest. Specifically, consider an SDP in standard form
\begin{align*}
&\underset{\Ub}{\text{minimize}}~ \trip{\Db}{\Ub}\\
&\text{s.t.}~~ \Ub \in \cC := \{ \Ub : \trip{\Ab_i}{\Ub} = b_i \text{ for } i=1,\dots,p;~\Ub \succeq 0\}, \numberthis \label{eqn:sdp_form}
&\end{align*}
where $\Ab_i \in \cS^{n \times n}$, $\Db \in \cS^{n \times n}$ and $b_i \in \RR$; denote the optimal value of \eqref{eqn:sdp_form} by $u^*$. To apply RSS, we must specify as input any strictly feasible solution $\Fb$ to \eqref{eqn:sdp_form}.\footnote{Actually \citet{Renegar2014} works in the setting $\Fb=\Ib$; what we present here is a slightly modified version and later we use the results of the corresponding, adjusted theoretical analysis}
Given $\Fb$, a projection can be defined from $\Fb$ onto the border of the positive semi-definite cone by $P_{\Fb}(\Ub) = \Fb + \frac{1}{1 - \lminf{\Ub}} \left(\Ub - \Fb\right)$, where $\lminf{\Ub} = \lmin{\Fb^{-1/2}\Ub\Fb^{-1/2}}$.
$P_{\Fb}(\Ub)$ lies at the intersection of the line segment between $\Fb$ and $\Ub$ and the positive semi-definite cone. Clearly if $\Ub \in \psdc{n}$ then $P_{\Fb}(\Ub) \in \psdc{n}$. Now, let $u_0 \in \RR$ satisfying $u_0 < \langle \Db,\Fb \rangle$. \citet[Theorem 2.2]{Renegar2014} shows that if $\Vb^*$ is a global optimum for
\begin{align*}
&\underset{\Vb}{\text{maximize}}~\lminf{\Vb}\\
&\text{s.t.}~~ \Vb \in \cC_{\lambda}:= \{\Vb : \trip{\Ab_i}{\Vb} = b_i \text{ for } i=1,\dots,p;~\trip{\Db}{\Vb} = u_0 \} \numberthis \label{eqn:lmin_form}
\end{align*}
then $P_{\Fb}(\Vb^*)$ is optimal for \eqref{eqn:sdp_form}. In addition, if $\Ub^*$ is optimal for \eqref{eqn:sdp_form}, then
$\Vb^* = \Fb + \frac{\trip{\Db}{\Fb} - u_0 }{\trip{\Db}{\Fb} - u^*}\left(\Ub^* - \Fb \right)$ is optimal for \eqref{eqn:lmin_form}. To obtain faster convergence, Nesterov's smoothing technique can be applied and the objective function in \eqref{eqn:lmin_form} can be replaced by
\begin{equation}
f_{\mu,\Fb}(\Vb) = - \mu \log \sum_j \exp\left(-\lambda_{j}(\Fb^{-1/2} \Vb \Fb^{-1/2}) / \mu \right),
\end{equation}
giving the smoothed problem
\begin{align*}
&\underset{\Vb}{\text{maximize}}~f_{\mu,\Fb}(\Vb)\\
&\text{s.t.}~~ \Vb \in \cC_{\lambda}:= \{\Vb : \trip{\Ab_i}{\Vb} = b_i \text{ for } i\in[p];~\trip{\Db}{\Vb} = u_0 \}. \numberthis \label{eqn:lmin_form_smoothed}
\end{align*}
RSS internally applies Nesterov's accelerated projected gradient descent algorithm \citep{Bubeck2015} to \eqref{eqn:lmin_form_smoothed} several times through careful selection of initial iterates and after at most
\begin{equation}
\label{eqn:rss_convergence_rate_T}
T \leq 2R||\Fb^{-1}||_2^2\sqrt{\log d}  \left(\frac{1}{\epsilon} + \log_{5/4}\left(\frac{\trip{\Db}{\Fb} - u^*}{\trip{\Db}{\Fb} - u_0} \right)\right),
\end{equation}
updates, the matrix $\Ub_T$ output by RSS satisfies
\begin{equation}
\label{eqn:rss_convergence_rate_errr}
\frac{\trip{\Db}{\Ub_T} - u^*}{\trip{\Db}{\Fb} - u^*} \leq \epsilon.
\end{equation}
We direct the reader to \citet[Theorem 7.2]{Renegar2014} for additional details. To summarize -- applying RSS to an SDP requires strictly feasible $\Fb$, feasible $\Ub_0$ such that $\trip{\Db}{\Ub} < \trip{\Db}{\Fb}$, efficient computation of $\nabla f_{\mu,\Fb}$ and efficient computation of $\cP_{\cC_{\lambda}^{\perp}}$, the projection of the gradient onto $\cC_{\lambda}^{\perp} = \{\Ub | \trip{\Ab_i}{\Ub} = 0, \trip{\Db}{\Ub} = 0 \}$.

\subsubsection*{Conversion to an Eigenvalue Maximization Problem}
First, we introduce the augmented variables
\begin{equation}
\Ub' = 
\left[
\begin{array}{c|c}
\Ub & \bZero \\
\hline
\bZero & \Ub_{\cC}
\end{array}
\right],~~
\Ib_{a,b}' = 
\left[
\begin{array}{c|c}
\Ib_{a,b} & \bZero \\
\hline
\bZero & \frac{-1}{2}\diag(\eb_{a,b})
\end{array}
\right],
\label{eqn:augmented_variables}
\end{equation}
where $\Ub_{\cC}$ is a $d^2 \times d^2$ diagonal matrix of slack-variables and $\eb_{a,b}$ denotes the $d^2$-dimensional vector of 0s with 1s in only the $((a-1)d+b)^{th}$ and $((b-1)d+a)^{th}$ positions. We also define the variables $\Rb'_{a}$, $\Ib'$, and $\Db'$ as $(d^2+d)\times(d^2+d)$ matrices with upper left block equal to $\Rb_a$, $\Ib$ and $\Db$, respectively, and zero elsewhere. Up to the sign of the optimal value, \eqref{eqn:kmeans_sdp} can thus be expressed as
\begin{align*}
&\underset{\Ub'}{\text{minimize}}~ \trip{\Db' }{\Ub'}, \\
&\text{s.t.}~~ \Ub' \in \cC := \left\{\begin{array}{l l l}
\multirow{2}{*}{$\Ub':$} & \trip{\Ib'_{ab}}{\Ub'} = 0 \text{ for } a \leq b;&\trip{\Rb'_a}{\Ub'} = 2 \text{ for all } a;\\
&\trip{\Ib'}{ \Ub'} = K;&\Ub' \succeq 0.
\end{array}\right\} \numberthis  \label{eqn:kmeans_sdp_sf} 
\end{align*}
Given a strictly feasible solution $\Fb$ and $\Ub_0$ such that $\trip{- \Db}{\Fb} < \trip{- \Db}{\Ub_0} = - u_0$ to \eqref{eqn:kmeans_sdp}, we construct the pair
\[
\Fb' = 
\left[
\begin{array}{c|c}
\Fb & \bZero \\
\hline
\bZero & \dvect(\Fb)
\end{array}
\right],
\Ub_0' = 
\left[
\begin{array}{c|c}
\Ub_0 & \bZero \\
\hline
\bZero & \diag(\vect(\Ub_0))
\end{array}
\right]
\]
necessary to apply RSS to \eqref{eqn:kmeans_sdp_sf}. Finally, turning \eqref{eqn:kmeans_sdp_sf} into an equivalent eigenvalue maximization problem and applying Nesterov's smoothing gives
\begin{align*}
&\underset{\Vb'}{\text{maximize}}~~ f_{\mu,\Fb'}(\Vb'), \\
&\text{s.t.}~~ \Vb' \in \cC_{\lambda} := \left\{\begin{array}{l l l}
\multirow{2}{*}{$\Vb':$} & \trip{\Ib'_{ab}}{\Vb'} = 0 \text{ for } a \leq b;&\trip{\Rb'_a}{\Vb'} = 2 \text{ for all } a;\\
&\trip{\Ib'}{ \Vb'} = K;&\trip{\Db'}{ \Vb'} = u_0.
\end{array}\right\} \numberthis \label{eqn:kmeans_lmin_form}
\end{align*}
Importantly, we note that
\begin{equation}
\label{eqn:lmin_slackvars}
\lambda\bigl(\Fb^{'-\frac{1}{2}}\Vb' \Fb^{'-\frac{1}{2}}\bigr) = \lambda\bigl(\Fb^{-\frac{1}{2}}\Vb \Fb^{-\frac{1}{2}}\bigr)\bigcup \bigl\{X_{i,j}/F_{i,j}^{-1} \bigr\}.
\end{equation}

\subsubsection*{Constraint Set Projection}
To project onto $\cC_{\lambda}^{\perp}$, we must find the optimizer for $\cP_{\cC_{\lambda}^{\perp}}(\Ub')$. Notationally, $(U_{\cC})_{a,b}$ refers to the $((a-1)d+b)^{th}$ diagonal entry in $\Ub_{\cC}$ as it is a diagonal matrix of the slack variables. Because a projection onto a convex set has a unique minimizer, it suffices to find any point satisfying the KKT conditions. Solving for the projection gives the following system of $d+2$ equations in $d+2$ unknowns:
\begin{align*}
\sum_{b = 1}^d U_{ab} + \sum_{b = 1}^d (U_{\cC})_{ab} &= \sum_{b=1}^d y^*_b + d y^*_a + y^*_T + \left[\sum_{b = 1}^d D_{ab} \right]\lambda^* & \text{ for } a \in [d] \\
\tr(\Ub)+ \tr(\Ub_{\cC}) &= \sum_{a=1}^d y^*_a + d y^*_T + \tr(\Db)\lambda^* & \\
\tr(\Db\Ub) + \tr(\Db\Ub_{\cC}) &= 2\sum_{a=1}^d \left[\sum_{b=1}^d D_{ab} \right] y^*_a + \tr(\Db) y^*_T + \tr(\Db\Db) \lambda^*. & \numberthis \label{eqn:kmeans_projection_system}
\end{align*}
Solving \eqref{eqn:kmeans_projection_system}, we get the projected matrix 
\begin{equation}
\cP_{\cC_{\lambda}^{\perp}}(\Vb_*') = \left[
\begin{array}{c|c}
\Vb_* & \bZero \\
\hline
\bZero & \dvect(\Vb_*)
\end{array}
\right],
\Vb_* = \frac{1}{2}\left[\Ub + \Ub_{\cC} - \sum_{a=1}^d \Rb_a y_a^* - y^*_T\Ib - \lambda^*\Db \right].
\label{eqn:kmeans_projection_result}
\end{equation}

\begin{remark}
The last two sections highlight how the {\it effective} dimension of the problem is reduced by conversion to an eigenvalue maximization problem. The $d^2$ slack variables do not affect the cost of computing the projection $\cP_{\cC_\lambda^\perp}$. Likewise \eqref{eqn:lmin_slackvars} shows that the cost of evaluating $f_{\mu,\Fb'}$ is dominated by that of computing the eigenvalues of the upper $d\times d$ diagonal block.
\end{remark}

\subsubsection*{Existence of a Strictly Feasible Solution}
Unlike for the SDP's considered by \citet{Renegar2014}, $\Ib$ is not feasible for \eqref{eqn:kmeans_sdp} as  $K < d$, $\tr(\Ib) = d \neq K$. We also note that the intuitive idea to find a possibly suboptimal clustering $\hat G$ and use $\Fb=B(\hat G)$ is not possible because  strict feasibility for \eqref{eqn:kmeans_sdp} requires all $\Fb_{ij} > 0$.

Nonetheless, there are valid choices of $\Fb$. Consider matrices of the form $\Fb = a\Ib + b\bone\bone^T$, where $a, b > 0$. Such matrices clearly satisfy $\Fb_{ij} > 0$ and $\Fb \succ 0$, so all that remains is to choose $a$ and $b$ such that $\trip{\Fb}{\Ib} = K$ and $\Fb\bone = \bone$. Multiplying these expressions out, simplifying and solving the resulting system of equations gives $a = \frac{K-1}{d-1}$ and $b = \frac{d-K}{d^2-d}$. Lemma \ref{lem:strictly_feasible_existence} summarizes the properties of $\Fb$.
\begin{lemma}
\label{lem:strictly_feasible_existence}
Given $d$ and $K$, define
\[
\Fb_{d,K} := \frac{K-1}{d-1} \Ib + \frac{d-K}{d^2-d} \bone \bone^T.
\]
$\Fb_{d,K}$ is strictly feasible for \eqref{eqn:kmeans_sdp} and $||\Fb^{-1}||_2 = \frac{d-1}{K-1}$.
\end{lemma}
\begin{proof}[Proof of Lemma \ref{lem:strictly_feasible_existence}]
The first claim follows by the previous discussion and the second follows immediately from Lemma \ref{lem:F_props}.
\end{proof}

\subsection{FORCE Algorithm: Dual Step}
\label{sec:FORCE_dual}
Because all instances of \eqref{eqn:kmeans_sdp} are strictly feasible, as shown in Lemma \ref{lem:strictly_feasible_existence}, then for any primal optimal solution there exists a dual solution such that its objective value is exactly equal to the primal. Unlike the primal problem, however, the dual does not lend itself easily to mapping a clustering onto a feasible solution for the SDP.

Let $\hat{G} = \{ \Gh{1},\dots,\Gh{K} \}$ be the candidate clustering for which we want to find a dual solution. Because the goal is to certify optimality, consider $\hat{G} = G^*$. Without loss of generality we can assume that the variables are ordered according to $G^*$, so that $\Bb^* = B(G^*)$ is block-diagonal. Denote by $d^* = \trip{\Db}{\Bb^*}$ and $\Qb := \sum_{a=1}^d \yah{a} \Rb_a + \yTh \Ib + \Db - \sum_{a \leq b} \yabh{a}{b} \Ib_{a,b}$. Complementary slackness gives that for $a \in \Gs{i}$ and $b \in \Gs{i}$, $\yabh{a}{b} = 0$. Thus if we can ``eliminate'' the off-diagonal blocks in $\Qb$, finding a dual solution should be very straightforward; this motivates Property \ref{asm:large_diagonal_blocks}.
\begin{property}[Large Diagonal Blocks Property]
\label{asm:large_diagonal_blocks}
An instance $\Db$ of a clustering problem satisfies the Large Diagonal Blocks Property if there exists a feasible dual solution with value $d^*$ such that the variables $\yabh{a}{b}$ can be chosen to make off-diagonal blocks of the matrix $\Qb$ equal to $\bZero$.
\end{property}
Intuitively, we expect that in the variable clustering setting Property \ref{asm:large_diagonal_blocks} will frequently hold. Because $- \Db$ is  an estimate of a covariance matrix for a generative model with block covariance structure, the diagonal blocks should dominate the off-diagonal blocks.
What remains then is to search over assignments to $\yah{a}$ and $\yTh$. In light of Lemma \ref{lem:dual_sdp_sol}, the FORCE dual solution construction can be viewed as a function of $y_T$:
\begin{equation}
\yb_{\Gs{i}}(\Db,y_T) = \Lb_i(-\Db_i\bone - y_T\bone), ~y_{a,b}(\Db,y_T) = \begin{cases}
0, \text{ if } a = b \\
y_a + y_b + D_{a,b}, \text{ o/w,}
\end{cases}
\label{eqn:force_dual_certificate}
\end{equation}
where $\Lb_i = \Lb^{-1}_{\Gs{i},\Gs{i}}$ and $\Db_i =\Db_{\Gs{i},\Gs{i}}$. 
By performing binary search over $y_T$, we obtain such a feasible dual solution if and only if Property \ref{asm:large_diagonal_blocks} is satisfied. Computation of \eqref{eqn:force_dual_certificate} is straightforward using Lemma \ref{lem:L_matrix} below.
\begin{lemma}
\label{lem:L_matrix}
Let $\Lb$ be defined as above as in Section \ref{sec:FORCE_dual}. Then $\Lb$ is invertible and its inverse is block-diagonal, given by
$\Lb^{-1}_{\Gs{i},\Gs{i}} = \frac{1}{|\Gs{i}|}\Ib - \frac{1}{2|\Gs{i}|^2}\bone\bone^T$.
Furthermore, $\lambda_{\max}\bigl(\Lb^{-1}_{\Gs{i},\Gs{i}}\bigr) = |\Gs{i}|^{-1}$.
\end{lemma}

\begin{proof}
Using the Sherman-Morrison formula we can obtain the first claim, from which the second follows immediately.
\end{proof}

We set the search interval for $y_T$ to be $[0,C]$ for some $C$ that can be selected at runtime. In practice to select the bound $C$, we will see from the proof of Theorem \ref{thm:gblock_existence} in Section \ref{sec:FORCE_dual_theory} can select
\[
C = 2||\bGammaH||_{\infty}\left(\frac{d}{n} +\sqrt{\frac{d}{n}} \right).
\]
Under the conditions of the theorem, there exists with high probability (tending to 1 as $d\rightarrow\infty$) a $y_T \in [0,C]$ such that the corresponding dual certificate is a feasible solution for \eqref{eqn:kmeans_dual}. We note that in the statement of Theorem \ref{thm:gblock_existence} there is a constant $c_1$ which above we have absorbed into the probability term.

\subsection{Convergence Rate of FORCE}
Denoting by $O_C$ a rounding oracle (e.g. Lloyd's Algorithm or CLINK), $O_C$ a certificate oracle that returns a dual feasible tuple $(\yb_a,\yb_{a,b},y_T)$, and $h$ the dual certificate search frequency, we can combine the primal update and dual certificate giving FORCE as Algorithm \ref{alg:smoothed_primal_dual}.  On its own, the FORCE Primal Step offers an improved theoretical guarantee over second-order interior point methods for \eqref{eqn:kmeans_sdp}.  By appropriately choosing the dual certificate search frequency $h$, the convergence rate properties of the primal step transfer to FORCE. These results are summarized as Theorem \ref{thm:force_running_time}.

\begin{algorithm}
\caption{First-Order Certifiable Clustering (FORCE)}
\label{alg:smoothed_primal_dual}
\begin{algorithmic}
\Input $0 < \epsilon < 1$, $\Db$, $h$, $\Ub_0$, $\Fb$
\Output $\hat G$
\State Run RSS with inputs $\epsilon$, $\Db$, $\Ub_0$, $\Fb$ for $T$ steps, denoting the iterate at time $s$ by $\Vb_s$
\For{each update $s \in [T]$ such that $s \mod h == 0$}
    \State $\Ub_s \gets P_{\Fb}(\Vb_{s})$, $\hat G_s \gets O_R(\Ub_{s})$
    \State $(\yb_a,\yb_{a,b},y_T) \gets O_C(\hat G_S)$
    \State If $2\sum_{a=1}^d y_a + Ky_T ~\texttt{==}~ \trip{-\Db}{\Ub_s}$, then \Return $\hat G_s$
\EndFor
\State \Return $O_R(P_{\Fb}(\Vb_T))$
\end{algorithmic}
\end{algorithm}

\begin{theorem}
\label{thm:force_running_time}
Let $C$ and $h$ be selected such that $C/h \leq 1$. Then, Algorithm \ref{alg:smoothed_primal_dual} terminates after $\tilde\cO\left(d^6K^{-2}\epsilon^{-1}\right)$ arithmetic operations, giving an $\epsilon$-optimal solution.
\end{theorem}
\begin{proof}
We start by showing that the claim holds for RSS applied to \eqref{eqn:kmeans_sdp}. Note that for any $\Ub$ and $\Vb$ $\in \cC$, $||\Ub-\Vb||_F \leq \sqrt{2}d$. For $\Fb_{d,K}$, applying Lemma \ref{lem:strictly_feasible_existence} gives $||\Fb_{d,K}^{-1}||_2^2 = \frac{d-1}{K-1}$. The iteration complexity of RSS, \eqref{eqn:rss_convergence_rate_T}, gives that the number of gradient updates required is at most
\[
T = \left(2\sqrt{2\log d}\right)\frac{d(d-1)^2}{(K-1)^2}\left(\frac{1}{\epsilon} + \log_{5/4}\left(\frac{\trip{\Db}{\Fb} - u^*}{\trip{\Db}{\Fb} - u_0} \right) \right).
\]
From \eqref{eqn:lmin_slackvars}, computing the gradient of $f_{\mu,\Fb'}$ requires $\cO(d^3)$ arithmetic operations and from \eqref{eqn:kmeans_projection_result} we see that projecting the gradient likewise requires $\cO(d^3)$ operations. Therefore the running time of RSS is bounded by $\tilde\cO\left(d^6K^{-2}\epsilon^{-1}\right)$.

All that remains is to determine the cost of each query to the oracles $O_R$ and $O_C$. Using CLINK as $O_R$, $\cO(d^2)$ arithmetic operations are required per query.  For $O_C$,we observe that at most $\cO(C\log C)$ iterations of binary search are required. By pre-computing the transformations for $\yb_{\Gs{i}}$, which requires at most $\cO(d^3)$ arithmetic operations, each iteration of the search requires computing only the minimum eigenvalue of a $d$-dimensional matrix. This gives an overall bound of $\tilde{\cO}(Cd^3)$ on the number of arithmetic operations for $O_C$. Because there are at most $T/h$ calls to $O_C$ and we have that $C/h \leq 1$, the additional cost of all calls to $O_C$ is $\tilde\cO\left(d^6K^{-2}\epsilon^{-1}\right)$, concluding the proof.
\end{proof}

\section{Theoretical Properties of the Dual Certificate}
\label{sec:FORCE_dual_theory}
In the previous section, \eqref{eqn:force_dual_certificate} defined the FORCE dual certificate in terms of $y_T$. In this section, we state and prove Theorem \ref{thm:gblock_existence} showing that for variable clustering in $G$-Latent models, the certificate \eqref{eqn:force_dual_certificate} exists at $G^*$ w.h.p. whenever the cluster separation metric $\Delta\bCS$ is above a minimal threshold. Our approach is in keeping with the literature on analyzing statistical properties of SDP relaxations, and we use similar proof strategies \cite{Ames2014,Iguchi2015,Iguchi2015a,Awasthi2014}. Theorem \ref{thm:gblock_existence} also shows that the P-W SDP is tight for $G$-Latent models as whenever the certificate exists, the SDP must be tight.


\begin{theorem}
\label{thm:gblock_existence}
Consider the variable clustering setting under the $G$-Latent model and assume $\log d \leq p_0 n$, where $p_0$ is the constant from Section \ref{sec:preliminaries}. There exist constants $c_1$, $c_2$ and $c_3$ such that if
\[
\Delta\bCS \geq  c_1||\bGammaS||_{\infty}\left(\sqrt{\frac{\log d}{nm}} + \sqrt{\frac{d}{nm^2}} +\frac{d}{nm}  \right) + c_2\sigma\sqrt{\frac{\log d}{n}},
\]
then with probability at least $1-c_3/d$ the FORCE Dual Certificate exists at $G^*$, where $\sigma = \max_i \CS{i}{i} + ||\bGammaS||_\infty$.
\end{theorem}

\subsection{General Properties}
Denoting by $(\Db,G^*)$ an instance of \eqref{eqn:kmeans_sdp}, we now characterize the factors that determine when Property \ref{asm:large_diagonal_blocks} is satisfied -- when, for each $i$, $y_T$ can be selected such that (a) for all $a$ and $b$, $y_{a,b}(\Db,y_T) \geq 0$, and (b) that $\Qb_{i}(\Db,y_T) := \Db_{\Gs{i},\Gs{i}} + \sum_{a \in \Gs{i}}y_a \Rb_{a} + y_T \Ib$ is positive semidefinite.
Importantly, problem (b) requires studying the behavior of points or variables only within the same group, greatly simplifying the analysis. Lemma \ref{lem:q_eigenvalue_concentration} characterizes the behavior of the minimal eigenvalue of $\Qb_i$.
\begin{lemma}
\label{lem:q_eigenvalue_concentration}
Using the notation and quantities introduced above $\lmin{\Qb_{i}(\Db,y_T)} \\ =  y_T + \min\{-y_T,\lmin{\Qb_{i}^{\perp}(\Db)} \}$, where
\[
\Qb_{i}^{\perp}(\Db) := \frac{\left(\bone^T\Db_{\Gs{i},\Gs{i}}\bone \right)\bone\bone^T}{|\Gs{i}|^2} - \frac{\bone\bone^T \Db_{\Gs{i},\Gs{i}} + \Db_{\Gs{i},\Gs{i}}\bone\bone^T}{|\Gs{i}|} + \Db_{\Gs{i},\Gs{i}}.
\]
\end{lemma}
\begin{proof}
To demonstrate the result, we first find an expression of the minimal eigenvalue of $\Qb_{i}(\Db,y_T)$ in terms of $y_T$ and $\Db_{\Gs{i},\Gs{i}}$. Then we can apply Lemma \ref{lem:dg_eigenvalue_concentration} to obtain the result. One way to express the minimum eigenvalue is
\[
\argmin_{\bv \in \cS^{|\Gs{i}|-1}} \underlabel{\bv^T\Qb_{i}(\Db,y_T)\bv}{(i)}.
\]
Now, for any $\bv \in \cS^{|\Gs{i}|-1}$ we can expand (i) as
\begin{align*}
\text{(i)} &= \sum_{a=1}^{|\Gs{i}|}\sum_{b=1}^{|\Gs{i}|} v_a v_b Q_i(\Db,y_T)_{a,b}\\
&= \sum_{a=1}^{|\Gs{i}|}v_a^2 y_T  + \sum_{a=1}^{|\Gs{i}|}\sum_{b=1}^{|\Gs{i}|} v_a v_b (y_a + y_b) + \sum_{a=1}^{|\Gs{i}|}\sum_{b=1}^{|\Gs{i}|} v_a v_b D_{a,b}\\
&= y_T+ \underlabel{\bv^T \Db_{\Gs{i},\Gs{i}} \bv}{(ii.a)} + 2 \underlabel{\sum_{a=1}^{|\Gs{i}|}\sum_{b=1}^{|\Gs{i}|} v_a v_b y_a}{(ii.b)}. \numberthis \label{eqn:q_eig_expansion}
\end{align*}
Via some algebra we obtain
\[
\text{(ii.b)} = \sum_{a=1}^{|\Gs{i}|}v_a y_a \sum_{b=1}^{|\Gs{i}|} v_b = \sum_{a=1}^{|\Gs{i}|}v_a y_a \bv^T\bone = \bv^T\bone \yb_{\Gs{i}}^T \bv. \label{eqn:ii_c_algebra}
\]
From \ref{eqn:ii_c_algebra} above we see that the object of interest is now $\bone \yb_{\Gs{i}}^T$, a $|\Gs{i}| \times |\Gs{i}|$ matrix. Recall that $\yb_{\Gs{i}}^T$ is ultimately a function of $y_T$ and $\Db$. Fortunately, we already have explicit expressions for these quantities. In particular,
\begin{align*}
\bone \yb_{\Gs{i}}^T &= \bone\left(-\bone^Ty_T - \bone^T\Db_{\Gs{i},\Gs{i}}\right)\Lb^{-1}_{\Gs{i},\Gs{i}} \\
&= -y_T \bone\bone^T\Lb^{-1}_{\Gs{i},\Gs{i}} - \bone\bone^T\Db_{\Gs{i},\Gs{i}} \Lb^{-1}_{\Gs{i},\Gs{i}} \\
&= -y_T\frac{1}{|\Gs{i}|}\bone\bone^T + \frac{1}{2|\Gs{i}|}\bone\bone^T - \frac{1}{|\Gs{i}|}\bone\bone^T \Db_{\Gs{i},\Gs{i}} + \frac{1}{2|\Gs{i}|^2}\bone\bone^T  \Db_{\Gs{i},\Gs{i}}\bone\bone^T \\
&= -\frac{y_T}{2|\Gs{i}|}\bone\bone^T  - \frac{1}{|\Gs{i}|}\bone\bone^T \Db_{\Gs{i},\Gs{i}} + \underlabel{\frac{1}{2|\Gs{i}|^2}\bone\bone^T  \Db_{\Gs{i},\Gs{i}}\bone\bone^T}{(iii)}. \numberthis \label{eqn:q_eig1}
\end{align*}
In \ref{eqn:q_eig1}, observe that $\text{(iii)} = \frac{1}{2|\Gs{i}|^2}(\bone^T\Db_{\Gs{i},\Gs{i}}\bone)\bone\bone^T$. Plugging this back into \ref{eqn:q_eig1} gives that
\begin{equation}
\label{eqn:q_eig2}
\bone \yb_{\Gs{i}}^T = \frac{1}{2|\Gs{i}|^2} \left(\bone^T\Db_{\Gs{i},\Gs{i}}\bone - |\Gs{i}| y_T \right)\bone\bone^T - \frac{1}{|\Gs{i}|}\bone\bone^T \Db_{\Gs{i},\Gs{i}}.
\end{equation}
We can substitute \ref{eqn:q_eig2} into \ref{eqn:ii_c_algebra}, yielding that
\begin{align*}
\text{(ii.b)} &= \bv^T\left(\frac{\left(\bone^T\Db_{\Gs{i},\Gs{i}}\bone - |\Gs{i}| y_T \right)\bone\bone^T}{|\Gs{i}|^2} - \frac{\bone\bone^T \Db_{\Gs{i},\Gs{i}}}{|\Gs{i}|} - \frac{\Db_{\Gs{i},\Gs{i}}\bone\bone^T}{|\Gs{i}|} \right)\bv\\
&= \bv^T\left(\Qb_{i}^{\perp}(\Db) -\frac{y_T}{|\Gs{i}|}\bone\bone^T - \Db_{\Gs{i},\Gs{i}}\right)\bv.
\end{align*}
Substituting back into (i), we get that
\[
\lmin{\Qb_{i}(\Db,y_T)} = y_T + \lmin{-\frac{y_T}{|\Gs{i}|}\bone\bone^T + \Qb_{i}^{\perp}(\Db)}
\]
which is nearly the desired result. To proceed, we can see that $\frac{y_T}{|\Gs{i}|}\bone\bone^T$ and $\Qb_{i}^{\perp}(\Db)$ lie in orthogonal spaces. This is a deterministic statement and does not depend on any particular clustering instance. Indeed, we can check that
\[
\bone^T \Qb_{i}^{\perp}(\Db) \bone = 0
\]
This is good, because then their respective eigenspaces are orthogonal giving
\[
\lmin{\Qb_{i}(\Db,y_T)} =  y_T + \min\{-y_T,\lmin{\Qb_{i}^{\perp}(\Db)} \}.
\]
\end{proof}

\subsection{Properties under the G-Latent Model}
Now the setup is that we have $n$ samples of a $d$-dimensional random vector, denoted by $\Xb \in \RR^{n\times d}$,  $\Db$ is the PECOK penalized covariance estimator (Section \ref{sec:introduction}). By writing $\yb_{\Gs{i}}(\Xb,y_T)$ as a function of $\Xb \in \RR^{n\times d}$, it is easy to observe that $\EE[\yb_{\Gs{i}}(\Xb,y_T)] \approx \frac{1}{2}(\CS{i}{i} - |\Gs{i}|^{-1}y_T)\bone$ and therefore $\EE[\Qb_{i}(\Xb,y_T)] \approx y_T(\Ib - |\Gs{i}|^{-1}\bone\bone^T)$.\footnote{The equalities are inexact because we make no assumptions on the mean of $\hat\bGamma$, only its convergence rate.} From Lemma \ref{lem:q_eigenvalue_concentration}, whether or not the FORCE construction succeeds depends on how quickly $\Qb_{i}(\Xb,y_T)$ concentrates about its mean (in terms of spectral norm) which in turn determines if $y_T$ can be chosen small enough to ensure that the corresponding $y_{a,b}$ are feasible. Accordingly, the final two ingredients needed to prove Theorem \ref{thm:gblock_existence} are Lemma \ref{lem:dg_eigenvalue_concentration} which controls the spectral radius of $\Qb_{i}^{\perp}(\bX)$ and Lemma \ref{lem:gblock_yab} which bounds $y_{a,b}$ in terms of $y_T$.

\begin{lemma}
\label{lem:dg_eigenvalue_concentration}
Assume that $\log d \leq p_0 n$. Then,
\begin{equation*}
||\Qb_{i}^{\perp}(\bX)||_2 \leq c_1||\bGammaS||_{\infty}\left(\frac{d}{n} +\sqrt{\frac{d}{n}}  \right),
\end{equation*}
with probability at least  $1-\frac{c_2}{d^2}$ where $c_1$ and $c_2$ are constants.
\end{lemma}

\begin{lemma}
\label{lem:gblock_yab}
Let $i$ and $j$ be in $[K]$ and $i \neq j$. Define $y'_{a,b}(\bX,y_T) := D_{a,b} +y_a + y_b$ for all $a \in \Gs{i}$ and $b \in \Gs{j}$. Under the assumption $\log d \leq p_0 n$,
\begin{align*}
y'_{a,b} \geq \frac{1}{2}\Delta(\bCS) - \frac{1}{m}y_T - c_1||\bGammaS||_{\infty}\sqrt{\frac{\log d}{n m }} - c_2\sigma\sqrt{\frac{\log d}{n}},
\end{align*}
with probability at least $1-c_3/d^3$, where $c_1$, $c_2$ and $c_3$ are constants, $\sigma = \max_i \CS{i}{i} + ||\bGammaS||_\infty$.
\end{lemma}

\subsubsection*{Proof of Theorem \ref{thm:gblock_existence}}
Now that we have all the necessary lemmas, we can prove the main result. First, we select
\[
y_T' := \max_i ||\Qb_{i}^{\perp}(\bX)||_2,
\]
ensuring that all $\Qb_i(\bX)$ are positive semidefinite. By Lemma \ref{lem:dg_eigenvalue_concentration} and taking the union bound over all $i \in [K]$, 
\[
y_T' \leq c_1'||\bGammaS||_{\infty}\left(\frac{d}{n} +\sqrt{\frac{d}{n}} \right),
\]
with probability at least $1-c_2'/d$.

Furthermore, by taking the union bound over all $a$ and $b$ not in the same group and using Lemma \ref{lem:gblock_yab}, 
\begin{align*}
\min y'_{a,b} &\geq \frac{1}{2}\Delta(\bCS) - \frac{1}{2m}y_T - \frac{1}{2m}y_T - c_2''||\bGammaS||_{\infty}\sqrt{\frac{\log d}{n m }} - c_3''\sigma\sqrt{\frac{\log d}{n}}
\end{align*}
with probability at least $1 -c_1''/d$. Therefore, there exist constants $c_1$, $c_2$ and $c_3$ such that if we take $y_T = y_T'$ and 
\begin{align*}
\Delta\bCS \geq  c_1||\bGammaS||_{\infty}\left(\sqrt{\frac{\log d}{nm}}  + \sqrt{\frac{d}{nm^2}} +\frac{d}{nm}  \right) + c_2\sigma\sqrt{\frac{\log d}{n}},
\end{align*}
then with probability at least $1-c_3/d$, $\min_{a,b} y_{a,b} \geq 0$, demonstrating dual feasibility. Thus  with probability at least $1-c_3/d$, $y_T'$ gives a feasible solution to \eqref{eqn:kmeans_dual}, concluding the proof of the theorem.

\section{Extension of FORCE to Unknown K}
\label{sec:FORCE_unknown_K}
The motivation and insight behind the FORCE algorithm remains the same when $K$ is unknown, so we do not repeat the full discussion given in Section \ref{sec:FORCE_all}. When $K$ is not known a priori, it can sometimes be estimated simultaneously by exchanging the trace constraint for an appropriately chosen trace penalty. In the variable clustering setting, \citet{Bunea2016} show that \eqref{eqn:kmeans_adapt_sdp} recovers the optimal solution to \eqref{eqn:kmeans_nphard} without requiring $K$ to be known a priori at the same cluster separation rate as the setting where $K$ is known.

\subsubsection*{K-means Adaptive SDP}
We  refer to
\begin{equation}
\underset{\Ub}{\text{maximize}} \trip{- \Db - \hat \kappa \Ib}{\Ub}  ~~\text{s.t.}~~ \Ub \in \cC := \{ \Ub : \Ub \geq 0; \Ub \bone = \bone; \Ub \succeq 0\},
\label{eqn:kmeans_adapt_sdp}
\end{equation}
as the $K$-means Adaptive SDP due to its use in {\it adaptively} selecting the number of clusters and finding the optimal clustering simultaneously. The trace penalty is defined by a data driven tuning parameter $\hat \kappa$. It is beyond the scope of this work to consider the theoretical properties of \eqref{eqn:kmeans_adapt_sdp} for data clustering and the remainder of this section focuses on the variable clustering setting.

Like the case when $K$ is known, the dual SDP has the form
\begin{equation}
\begin{aligned}
& \underset{y_{a,b},y_a}{\text{minimize}} & & 2\sum_{a=1}^d y_a  \\
& \text{subject to} & & \sum_{a=1}^d y_a \Rb_a + \hat\kappa \Ib \succeq - \Db + \sum_{a \leq b} y_{a,b} \Ib_{a,b} \\
&&& y_{a,b} \geq 0 \text{ for all } a \leq b.
\end{aligned}
\label{eqn:kmeans_adapt_dual}
\end{equation}

\subsubsection*{Conversion to Eigenvalue Maximization}
The conversion to standard form and an eigenvalue maximization problem is nearly identical to the case when $K$ is known, so the derivation is omitted. Using the notation from Section \ref{sec:FORCE_all}, in standard form \eqref{eqn:kmeans_adapt_sdp} becomes 
\begin{align*}
&\underset{\Ub'}{\text{minimize}}~ \trip{\Db' + \hat\kappa \Ib'}{\Ub'}, \numberthis \label{eqn:kmeans_adpat_sdp_sf} \\
& \text{s.t.}~~ \Ub' \in \{\Ub' : \trip{\Ib'_{ab}}{\Ub'} = 0 \text{ for } a \leq b;~\trip{\Rb'_a}{\Ub'} = 2 \text{ for all } a;~  \Ub' \succeq 0\},
\end{align*}
with corresponding smoothed eigenvalue maximization problem 
\begin{align*}
&\underset{\Vb'}{\text{maximize}}~~ f_{\mu,\Fb'}(\Vb'), \\
&\text{s.t.}~~ \Vb' \in \cC_{\lambda} := \left\{\begin{array}{l l l}
\multirow{2}{*}{$\Vb':$} & \trip{\Ib'_{ab}}{\Vb'} = 0 \text{ for } a \leq b;&\trip{\Db'}{ \Vb'} = u_0;\\
&\trip{\Rb'_a}{\Vb'} = 2 \text{ for all } a.&
\end{array}\right\} \numberthis  \label{eqn:kmeans_adapt_lmin_form}
\end{align*}

\subsubsection*{Constraint Projection}
As in the case when $K$ is known, we must derive the projection onto $\cC_{\lambda}^{\perp}$. Solving the KKT conditions, we get the projected matrix
\begin{equation*}
\cP_{\cC_{\lambda}^{\perp}}(\Vb_*') = \left[
\begin{array}{c|c}
\Vb_* & \bZero \\
\hline
\bZero & \dvect(\Vb_*)
\end{array}
\right],
\Vb_* = \frac{1}{2}\left[\Ub + \Ub_{\cC} - \sum_{a=1}^d \Rb_a Y_a^* - \lambda^*\left(\Db + \hat\kappa \Ib \right) \right].
\end{equation*}

\subsubsection*{Existence of a Feasible Solution}
Clearly for \eqref{eqn:kmeans_adapt_sdp} $\Fb = \Ib$ is feasible, but unfortunately it is not strictly feasible so a different choice of $\Fb$ is required. Unlike in the case when $K$ is known, there is no trace constraint and therefore we can find an $\Fb$ such that for any $d$, $c_1^{-1} \leq \lmin{\Fb} \leq \lmax{\Fb} \leq c_1$, for some $c_1 \geq 1$. In particular, we can choose
\[
\Fb := \frac{1}{2}\Ib + \frac{1}{2d}\bone\bone^T,
\]
which clearly is strictly feasible for \eqref{eqn:kmeans_adapt_sdp}. Using the Sherman-Morrison formula, we obtain that
\[
\Fb^{-1} = 2\Ib - \frac{1}{d}\bone\bone^T.
\]
Furthermore, it is easy to see that $\frac{1}{2} \leq \lmin{\Fb} \leq \lmax{\Fb} \leq 2$. This shows that in the case where $K$ is unknown, we pay only a factor of 4 penalty for $\Ib$ not being strictly feasible. This is a sharp contrast to the fixed $K$ case, where the penalty is much higher.

\subsection{FORCE Dual Step}
In order to find a dual certificate, we first characterize the form of optimal solutions to \eqref{eqn:kmeans_adapt_dual}. Lemma \ref{lem:dual_char_adapt} characterizes all primal, dual optimal pairs for \eqref{eqn:kmeans_adapt_sdp}, just as Lemma \ref{lem:dual_sdp_sol} does for the case where $K$ is known a priori.
\begin{lemma}
\label{lem:dual_char_adapt}
The following are equivalent: (a) $\Bb^*$ is an optimal solution to \eqref{eqn:kmeans_sdp}, (b) every solution to \eqref{eqn:kmeans_dual} satisfies $y_{a,b} = 0$  for $a,b \in \Gs{i}$ and $\Qb_{\Gs{i},\Gs{i}}\bone = 0$  for all $i$, and (c) every solution to \eqref{eqn:kmeans_dual} satisfies $\yb_{\Gs{i}} = \Lb^{-1}_{\Gs{i},\Gs{i}}(-\Db_{\Gs{i},\Gs{i}}\bone - \hat\kappa\bone)$.
\end{lemma}
\begin{proof}
The proof of Lemma \ref{lem:dual_char_adapt} follows from complementary slackness and by re-arranging a system of linear equations. For more details, we direct the reader to \citet[Theorem 4]{Iguchi2015a}.
\end{proof}

Now, observe that in \eqref{eqn:kmeans_adapt_dual}, $\hat \kappa$ plays the same role as $y_T$ in \eqref{eqn:kmeans_dual}. Therefore the results and intuition regarding the dual construction still hold, but now there is no search over $y_T$. Instead we just invert a linear system and check feasibility.
The dual solution to \eqref{eqn:kmeans_adapt_dual} corresponding to  $G^*$ is
\begin{equation}
\label{eqn:FORCE_dual_solution_unknown_K}
\yb_{\Gs{i}}(\Db) = \Lb_i(-\Db_{i}\bone - \hat\kappa \bone), ~y_{a,b}(\Db) = \begin{cases}
0, \text{ if } a = b \\
y_a + y_b + D_{a,b}, \text{ o/w,}
\end{cases}
\end{equation}
where $\Lb_i = \Lb^{-1}_{\Gs{i},\Gs{i}}$ and $\Db_i =\Db_{\Gs{i},\Gs{i}}$. 
Just as the case when $K$ is known, we can use the explicit dual solution construction \eqref{eqn:FORCE_dual_solution_unknown_K} to certify the optimality.

\subsection{The FORCE Algorithm}
Algorithm \ref{alg:smoothed_primal_dual} requires only minor modification to be applied to \eqref{eqn:kmeans_adapt_sdp}. First, we apply RSS to \eqref{eqn:kmeans_adapt_sdp} instead of \eqref{eqn:kmeans_sdp}. Second we replace the certificate oracle $O_C$ with one based on \eqref{eqn:kmeans_adapt_dual}. Finally, we replace the rounding oracle $O_R$ with a procedure that can simultaneously cluster the projected iterate and select $K$. One such approach is to choose $K = \mathrm{round}(\tr(P_{\Fb}(\Vb_{s})))$ and then proceed by applying either CLINK or Lloyd's algorithm using the selected $K$. However, although this approach is theoretically justified, in practice one could consider using CLINK for the clustering step to obtain the entire solution path for all $K$, requiring only $\cO(d^2)$ arithmetic operations. The mean-squared error (MSE) of each clustering solution can be plotted against $K$ and the elbow method used to select $K$.  

\subsection{Theoretical Results}
Mirroring our results for fixed $K$, Theorem \ref{thm:kmeans_adapt_sdp_FORCE} gives a worst-case bound on the computational complexity of FORCE for \eqref{eqn:kmeans_adapt_sdp}. Proofs of the results in this section are nearly identical to those in Sections \ref{sec:FORCE_all} and \ref{sec:FORCE_dual_theory}.
\begin{theorem}
\label{thm:kmeans_adapt_sdp_FORCE}
For any certificate search frequency $h$, Algorithm \ref{alg:smoothed_primal_dual} applied to solving \eqref{eqn:kmeans_adapt_sdp} terminates in at most $\tilde{\cO}(d^{4}\epsilon^{-1})$ arithmetic operations, giving an $\epsilon$-optimal solution.
\end{theorem}

Next we address how to choose $\hat \kappa$ in practice. The choice is driven by the following consideration: when does the dual certificate exist and when is the SDP relaxation tight? These questions are intimately connected, and so similar to \citet{Bunea2016} we choose
\[
\hat\kappa := 5||\bGammaH||_{\infty}\left(\frac{d}{n} + \sqrt{\frac{d}{n}} \right)
\]
for variable clustering in $G$-Latent models when $K$ is unknown. As is made clear below, the choice of constant in $\hat \kappa$ could be altered, but we do not explore whether or not some other choice is preferable. Importantly $\hat \kappa$ is data-driven in the sense that it's selection requires {\it no knowledge} of the parameters of the generating distribution.

\begin{theorem}
\label{thm:gblock_adapt_existence}
Consider the variable clustering setting under the $G$-Latent model and assume $\log d \leq p_0 n$, where $p_0$ is the constant from Section \ref{sec:preliminaries}. If $\hat\kappa = 5||\bGammaH||_{\infty}\left(d/n + \sqrt{d/n} \right)$, there exist constants $c_1$, $c_2$ and $c_3$ such that if
\[
\Delta\bCS \geq  c_1||\bGammaS||_{\infty}\left(\sqrt{\frac{\log d}{nm}}  + \sqrt{\frac{d}{nm^2}} +\frac{d}{nm}  \right) + c_2\sigma\sqrt{\frac{\log d}{n}},
\]
then with probability at least $1-c_3/d$ the FORCE Dual Certificate exists at $G^*$, where $\sigma = \max_i \CS{i}{i} + ||\bGammaS||_\infty$.
\end{theorem}

\begin{remark}
The additional cost of constraining $K$ to be fixed is imposed directly by the trace constraint. It is somewhat surprising that we should obtain a significantly better worst-case complexity bound, for certain K, when we have {\it less information} about the structure of the problem at hand. For this reason we suspect it may not be impossible to obtain the same worst case bound if we impose a fixed $K$ in the problem formulation.
\end{remark}

The adaptive formulation, \eqref{eqn:kmeans_adapt_sdp}, can also be applied to data clustering, and we suspect the FORCE algorithm may have strong theoretical properties in that setting when $K$ is unknown, but that analysis is beyond the scope of this work.

\section{Numerical Results}
\label{sec:numerical}
We evaluate FORCE by validating Theorem \ref{thm:gblock_existence} empirically, comparing the FORCE primal step to other methods for solving \eqref{eqn:kmeans_sdp}, and comparing the performance of FORCE with clustering heuristics. Due to space constraints we focus on the case where $K$ is known, but similar results are obtained for $K$ unknown. Note that the third evaluation captures a combination of the properties of \eqref{eqn:kmeans_sdp} and of FORCE, since it is an inexact solver for the SDP.


\subsubsection*{Implementation Details} We implement FORCE in R and because FORCE is not a traditional primal-dual algorithm and does not make dual updates, we use an early stopping rule as the termination condition. Specifically, for a given $s$ and $\delta$, if at any iteration $t$, 
\[
\max_{u\in[t-s+1,t]} \frac{f_{\mu,\Fb}(\Vb_{u}) - f_{\mu,\Fb}(\Vb_{t-s})}{f_{\mu,\Fb}(\Vb_{t-s})} <  \delta,
\]
then the algorithm terminates. For all experiments we use $(s,\delta)=(100,10^{-4})$. An adaptive restart rule is used for the accelerated PGD weighting coefficients \citep{ODonoghue2015}. In practice, we also found that the warm-start step of RSS was unnecessary to achieve good performance, and the following simple heuristic gave at least as good results in terms of the final output: let $\Ub_0 = \frac{1}{d}B(\cK(\Db,K)) + \frac{d-1}{d}\Fb$, then perform accelerated PGD on $f_{\mu,\Fb}$ starting with initial iterate $\Vb_0 := \Ub_0$ for a fixed number of iterations $N$ to obtain $\Ub_1 := P_{\Fb}(\Vb_N)$. The matrix $\Ub_1$ is then used in place of the original warm-start step of RSS. We found that this heuristic produced in practice a matrix $\Ub_1$ satisfying the warm-start requirements of RSS.

\subsubsection*{Benchmarking Framework}
To benchmark the algorithms we use a Dell XPS 9570 with an i7-8750H processor. All algorithms are limited to 6 computational threads and the R build is linked against Intel's MKL BLAS implementation to ensure a fair comparison with MATLAB.

We compare FORCE with several alternatives. Primarily this shows how several alternative algorithms scale. We compare against a MATLAB implementation using MOSEK \citep{Andersen2000} as the solver, a MATLAB implementation using SDPNAL+ \citep{Sun2017}, and an ADMM algorithm to solve \eqref{eqn:kmeans_sdp} due to \citet{Ames2014}.\footnote{The authors have made the code available on-line at \url{http://bpames.people.ua.edu/software.html}} For short, we refer to these algorithms as MOSEK, SDPNAL+, and ADMM respectively. MOSEK and SDPNAL+ are run using the default options and ADMM is run using the same options as in \citet{Ames2014}. FORCE refers to Algorithm \ref{alg:smoothed_primal_dual} and FORCE-P denotes just the primal step of FORCE with no dual certificate search.

\subsubsection*{Generative Model}
Recall that the generating distribution of a $G$-Latent model with $d$ observed variables and $K$ latent factors can be described in terms of $(G^*,\bThetaS,\bGammaS)$. We first select a graph structure for $\Zb$ and then once the graph structure is constructed,  the latent precision matrix is defined as $\bThetaS = \rho\Wb + (|\lmin{\Wb}| + 0.2)\Ib$, where $\Wb$ is the adjacency matrix of the generated graph. We take $\bGammaS = \gamma\Ib$ for some constant $\gamma$ to be specified later. Because we work in the high-dimensional regime, we generate $n=d$ samples for each simulation.

Throughout we use the scale-free generative model to construct the dependency structure amongst the latent variables $\Zb$. It is a model for network data, whose degree distribution follows a power law and we generate the graph one node at a time, starting with a 2 node chain. For nodes $s \in \{3,\dots,K\}$, node $s$ is added and one edge is added between $s$ and one of the $s-1$ previous nodes. At each step, if $k_i$ denotes the current degree of node $i$ in the graph, the probability that node $t$ and node $i$ are connected is $p_i = k_i / (\sum_i k_i)$. By construction, such a graph always has $K$ edges.\footnote{Similar results can be obtained for other graph structures, such as Band or Hub graphs.}

\subsubsection*{Dual Certificate}
To assess the effect of noise on the existence of the dual certificates, \eqref{eqn:kmeans_dual} and \eqref{eqn:kmeans_adapt_dual}, we select two designs (one for $d=250$ and $d=500$) and vary the level of $\gamma$. Figure \ref{fig:force_dual_certificates} contains the results and we can observe a sharp phase transition as $\gamma$ increases. Interestingly only slightly less noise is required for the certificate to exist when $K$ is not known versus when $K$ is fixed a priori. Another interpretation of Figure \ref{fig:force_dual_certificates} is that it shows the sharp phase transition under which the P-W SDP is tight for $G$-Latent models as a function of noise $\gamma$.

\begin{figure}[H]
\centering
\begin{tabular}{l l}
\includegraphics[width=0.5\textwidth]{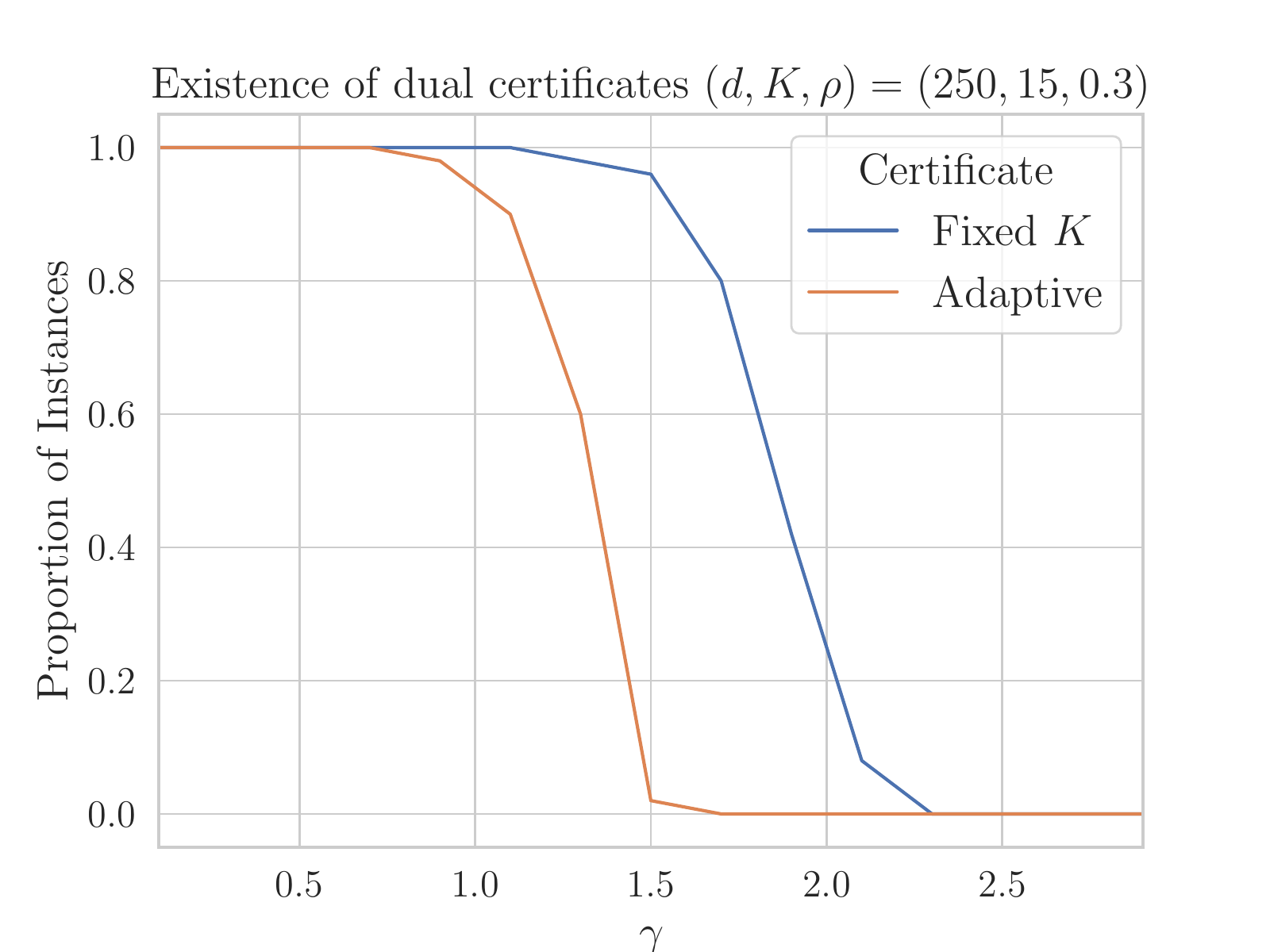} & \includegraphics[width=0.5\textwidth]{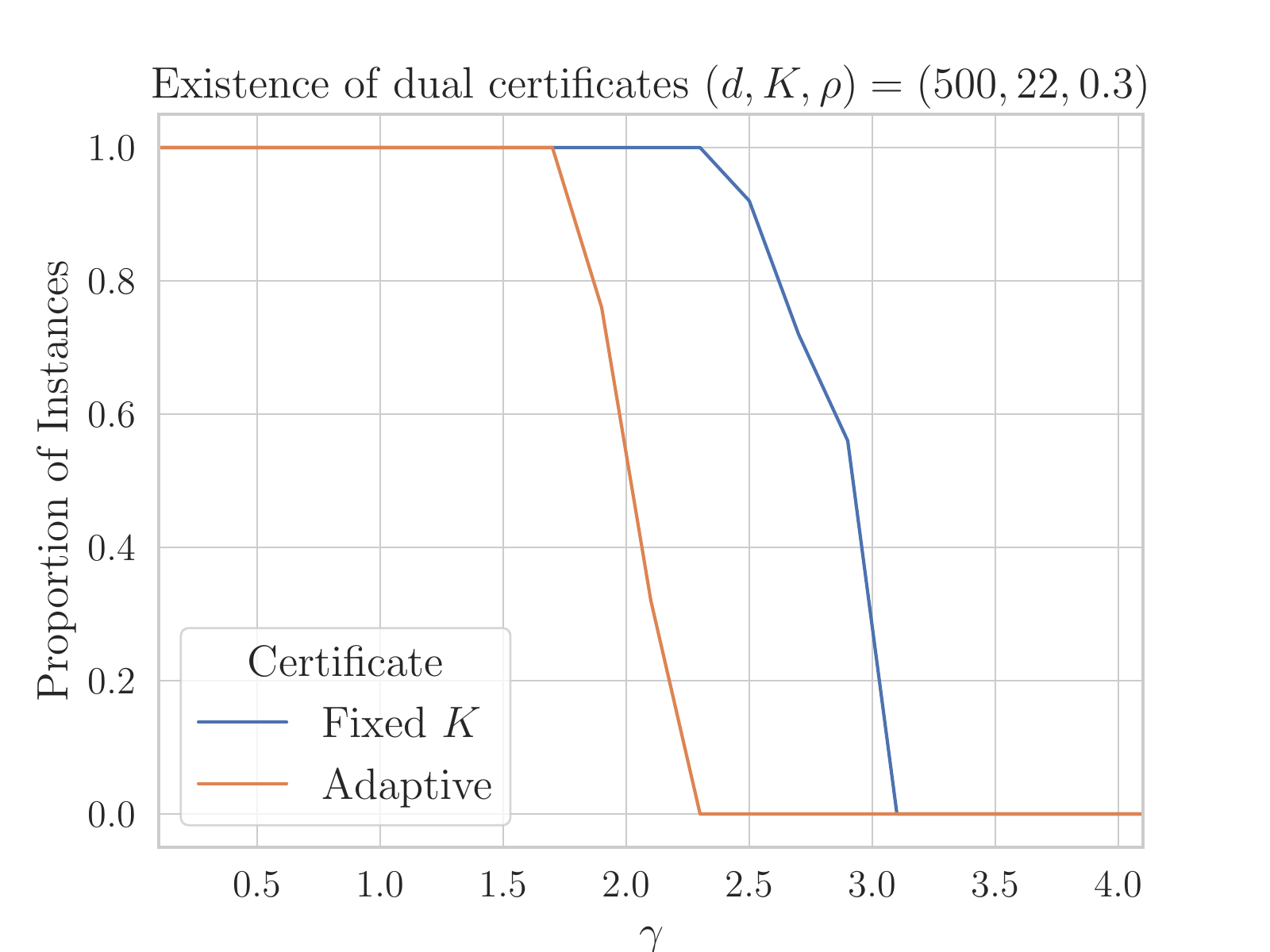}
\end{tabular}
\caption{Proportion of randomly generated instances for which a certificate exists at $G^*$.}
\label{fig:force_dual_certificates} 
\end{figure}

\subsection{FORCE vs. other algorithms for solving the P-W SDP}
\paragraph{Low-Dimensional Problem Sizes.}
The goal of the simulations in lower dimensions is to evaluate the scaling of the various alternatives. We vary both $d$ and the level of noise $\gamma$, evaluating six different settings. For each setting, we generate 100 random instances and run each algorithm. Relative error is measured in terms of the objective value of \eqref{eqn:kmeans_sdp} and we assume $v^* = \trip{-\Db}{B(G^*)}$.
Table \ref{tab:benchmark_lowd} gives the results, showing that even for $d=120$ MOSEK (a traditional interior point solver) is computationally expensive.

\begin{table}[htbp]
\centering
\caption{Benchmark results for low dimensional designs comparing FORCE and FORCE-P with MOSEK, SDPNAL+ and ADMM.}
\label{tab:benchmark_lowd}
\begin{tabular}{ l l l l l}
\toprule
$(d,k,\rho,\gamma)$ & \multicolumn{2}{c}{$(50,5,0.3,0.3)$} & \multicolumn{2}{c}{$(50,5,0.3,1.0)$}   \\
Algorithm & Rel. Err. & Time (sec)& Rel. Err. & Time (sec)\\ 
\midrule
MOSEK & $2.31\times 10^{-8}$ & $4.13\times 10^{-1}$ s & $3.82\times 10^{-8}$ & $4.24\times 10^{-1}$ s \\
SDPNAL+ & $5.55\times 10^{-7}$ & $3.97\times 10^{-1}$ s & $4.91\times 10^{-7}$ & $3.87\times 10^{-1}$ s \\
ADMM & $2.10\times 10^{-7}$ & $2.08\times 10^{-2}$ s & $3.07\times 10^{-7}$ & $3.29\times 10^{-2}$ s \\
FORCE & $0.00\times 10^{0}$ & $1.31\times 10^{-3}$ s & $1.10\times 10^{-8}$ & $1.97\times 10^{-2}$ s \\
FORCE-P & $6.86\times 10^{-3}$ & $9.18\times 10^{-2}$ s & $9.26\times 10^{-3}$ & $1.06\times 10^{-1}$ s \\
\midrule
\midrule
$(d,k,\rho,\gamma)$ & \multicolumn{2}{c}{$(50,5,0.3,0.3)$} & \multicolumn{2}{c}{$(50,5,0.3,1.0)$}   \\
Algorithm & Rel. Err. & Time (sec)& Rel. Err. & Time (sec)\\ 
\midrule
MOSEK & $1.04\times 10^{-8}$ & $3.42\times 10^{0}$ s & $2.56\times 10^{-8}$ & $3.58\times 10^{0}$ s \\
SDPNAL+ & $7.77\times 10^{-7}$ & $1.38\times 10^{0}$ s & $2.20\times 10^{-6}$ & $1.37\times 10^{0}$ s \\
ADMM & $1.42\times 10^{-7}$ & $7.48\times 10^{-2}$ s & $2.69\times 10^{-7}$ & $9.14\times 10^{-2}$ s \\
FORCE & $0.00\times 10^{0}$ & $2.65\times 10^{-2}$ s & $0.00\times 10^{0}$ & $8.85\times 10^{-2}$ s \\
FORCE-P & $7.60\times 10^{-3}$ & $3.54\times 10^{-1}$ s & $9.96\times 10^{-3}$ & $4.12\times 10^{-1}$ s \\
\midrule
\midrule
$(d,k,\rho,\gamma)$ & \multicolumn{2}{c}{$(50,5,0.3,0.3)$} & \multicolumn{2}{c}{$(50,5,0.3,1.0)$}   \\
Algorithm & Rel. Err. & Time (sec)& Rel. Err. & Time (sec)\\ 
\midrule
MOSEK & $2.17\times 10^{-8}$ & $1.83\times 10^{1}$ s & $3.96\times 10^{-8}$ & $1.83\times 10^{1}$ s \\
SDPNAL+ & $7.69\times 10^{-8}$ & $3.28\times 10^{0}$ s & $1.09\times 10^{-6}$ & $3.01\times 10^{0}$ s \\
ADMM & $1.07\times 10^{-7}$ & $4.41\times 10^{-2}$ s & $1.43\times 10^{-7}$ & $4.97\times 10^{-2}$ s \\
FORCE & $0.00\times 10^{0}$ & $4.62\times 10^{-2}$ s & $0.00\times 10^{0}$ & $1.84\times 10^{-1}$ s \\
FORCE-P & $7.25\times 10^{-3}$ & $7.18\times 10^{-1}$ s & $1.12\times 10^{-2}$ & $8.60\times 10^{-1}$ s \\
\bottomrule
\end{tabular}
\end{table}

\paragraph{High-Dimensional Problem Sizes.}
For higher dimensional setups, $d=500$, we find that both MOSEK and SDPNAL+ require too much memory and computational resources to run the simulations on our benchmarking platform (a high-end consumer PC), and therefore we compare only ADMM, FORCE and FORCE-P. We compare both high $(\gamma=3.0)$ and low $(\gamma=1.0)$ noise setups for $K=9,22,50,100$ which range from $\cO(\log d)$ to $\cO(d)$. For each design, 100 random instances were generated and the results are reported in Table \ref{tab:benchmark_highd}. To compute relative error we assume $v^* = \trip{-\Db}{B(G^*)}$.

When they converge, all three methods exhibit similar running times. However for designs closer to the threshold above which exact recovery is possible, ADMM often fails to converge. By comparison, FORCE and FORCE-P converge on all instances encountered during benchmarking. Table \ref{tab:benchmark_highd} shows that FORCE always has 0 error, i.e. that it achieves perfect recovery. From this we conclude FORCE-P finds a solution to \eqref{eqn:kmeans_sdp} that is ``close-enough'' to the optimal solution that by rounding and finding a dual certificate, FORCE achieves exact recovery.

Table \ref{tab:benchmark_highd} also reveals that as $K$ increases, it takes longer for FORCE to solve \eqref{eqn:kmeans_sdp}, which aligns with our intuition as the effective sample size per group is decreasing in $K$. This runs contrary, however, to the predictions of Theorem \ref{thm:force_running_time}. One reason Theorem \ref{thm:force_running_time} may be overly pessimistic is that the rate also depends on $||\Vb_0-\Vb^*||_\Fb$, the distance between the initial and optimal iterates. Our bound on this quantity may be too pessimistic in practice as we use a heuristic clustering to construct $\Vb_0$ and the heuristic should output a closer to optimal solution for smaller $K$ (indeed Figure \ref{fig:force_vs_heuristics} confirms this intuition).

\begin{table}[htbp]
\centering
\caption{Benchmark results for high dimensional designs comparing ADMM, FORCE and FORCE-P. $\cF$ is the event ADMM converges on a problem instance.}
\label{tab:benchmark_highd}
\begin{tabular}{ l l l l l l l}
\toprule
Alg.  & $(d,k,\rho,\gamma)$& Rel. Err. & Rel. Err.$|\cF$ & Conv. & Time (sec) \\ 
\midrule
ADMM & $(500,9,0.3,1.0)$ & $3.71\times 10^{-7}$ & $3.71\times 10^{-7}$ & $100.0\%$ & $2.39\times 10^{0}$ \\
FORCE &  & $0.00\times 10^{0}$ & $0.00\times 10^{0}$ & $100.0\%$ & $3.20\times 10^{-1}$ \\
FORCE-P &  & $9.12\times 10^{-3}$ & $9.12\times 10^{-3}$ & $100.0\%$ & $1.77\times 10^{1}$ \\
\midrule[0.1pt]
ADMM & $(500,9,0.3,3.0)$ & $5.38\times 10^{-7}$ & $5.38\times 10^{-7}$ & $96.0\%$ & $3.41\times 10^{0}$ \\
FORCE &  & $0.00\times 10^{0}$ & $0.00\times 10^{0}$ & $100.0\%$ & $1.29\times 10^{0}$ \\
FORCE-P &  & $2.39\times 10^{-2}$ & $2.40\times 10^{-2}$ & $100.0\%$ & $2.34\times 10^{1}$ \\
\midrule[0.1pt]
ADMM & $(500,22,0.3,1.0)$ & $1.86\times 10^{-7}$ & $1.86\times 10^{-7}$ & $100.0\%$ & $3.24\times 10^{0}$ \\
FORCE &  & $0.00\times 10^{0}$ & $0.00\times 10^{0}$ & $100.0\%$ & $4.03\times 10^{0}$ \\
FORCE-P &  & $1.70\times 10^{-2}$ & $1.70\times 10^{-2}$ & $100.0\%$ & $2.34\times 10^{1}$ \\
\midrule[0.1pt]
ADMM & $(500,22,0.3,3.0)$ & $7.99\times 10^{-7}$ & $7.99\times 10^{-7}$ & $56.0\%$ & $5.90\times 10^{0}$ \\
FORCE &  & $0.00\times 10^{0}$ & $0.00\times 10^{0}$ & $100.0\%$ & $8.48\times 10^{0}$ \\
FORCE-P &  & $2.29\times 10^{-2}$ & $2.29\times 10^{-2}$ & $100.0\%$ & $1.99\times 10^{1}$ \\
\midrule[0.1pt]
ADMM & $(500,50,0.3,1.0)$ & $2.96\times 10^{-8}$ & $2.96\times 10^{-8}$ & $96.0\%$ & $3.13\times 10^{0}$ \\
FORCE &  & $0.00\times 10^{0}$ & $0.00\times 10^{0}$ & $100.0\%$ & $1.14\times 10^{1}$ \\
FORCE-P &  & $1.69\times 10^{-2}$ & $1.72\times 10^{-2}$ & $100.0\%$ & $2.37\times 10^{1}$ \\
\midrule[0.1pt]
ADMM & $(500,50,0.3,3.0)$ & $5.84\times 10^{-8}$ & $5.84\times 10^{-8}$ & $64.0\%$ & $3.33\times 10^{0}$ \\
FORCE &  & $0.00\times 10^{0}$ & $0.00\times 10^{0}$ & $100.0\%$ & $1.46\times 10^{1}$ \\
FORCE-P &  & $2.11\times 10^{-2}$ & $1.99\times 10^{-2}$ & $100.0\%$ & $2.45\times 10^{1}$ \\
\midrule[0.1pt]
ADMM & $(500,100,0.3,1.0)$ & $1.32\times 10^{-8}$ & $1.32\times 10^{-8}$ & $20.0\%$ & $3.53\times 10^{0}$ \\
FORCE &  & $0.00\times 10^{0}$ & $0.00\times 10^{0}$ & $100.0\%$ & $1.56\times 10^{1}$ \\
FORCE-P &  & $1.16\times 10^{-2}$ & $1.08\times 10^{-2}$ & $100.0\%$ & $2.65\times 10^{1}$ \\
\midrule[0.1pt]
ADMM & $(500,100,0.3,3.0)$ & N/A & N/A & $0.0\%$ & N/A \\
FORCE &  & $0.00\times 10^{0}$ & N/A & $100.0\%$ & $2.57\times 10^{1}$ \\
FORCE-P &  & $2.08\times 10^{-2}$ & N/A & $100.0\%$ & $2.88\times 10^{1}$ \\
\bottomrule
\end{tabular}
\end{table}

\subsection{FORCE and the P-W SDP vs. Heuristic Methods}
Lastly, we compare FORCE applied to the P-W SDP to heuristic methods to cluster the data. Heuristic methods are typically fast, and if they were to offer similar performance in practice, it may not make sense to solve the P-W SDP using FORCE or any other algorithm. As we show in the experiments described below, this is not the case. We compare against Lloyd's algorithm with kmeans++ initialization as this gave better results than either CLINK or Lloyd's algorithm with random initialization. We consider the design $(d,K,\rho)=(500,22,0.3)$ and study the effect of $\gamma$ on the performance gap of FORCE and the P-W SDP versus heuristic methods.

First we compare clustering applied $\Vb_T$, the final iterate output by FORCE-P, to clustering applied to either $\Db = \bSigmaH-\bGammaH$ or $\bSigmaH$. Denoting by $\cK(\Mb,K)$ the algorithm that takes matrix $\Mb$ and runs Lloyd's algorithm with kmeans++ initialization returning a partition $\hat G$. The metrics used to evaluate the output are $d_1(\hat G, G^*) = \II[\hat G = G^*]$ and
$d_2(\hat G,G^*) = n^{-1}\sum_{i=1}^K \max_j |\hat G_i \cap G^*_j|,$
which captures the number of correctly assigned variables. 

Row one in Figure \ref{fig:force_vs_heuristics} shows $\EE[d_i(\cK(\Mb,K), G^*)]$ plotted against $\gamma$ for $\Mb=\bSigmaH,~\bSigmaH-\bGammaH,~ P_{\Fb}(\Vb_T)$;  the expectation is both with respect to the generating model and the randomness of $\cK$. For each level of $\gamma$, 50 random instances were generated, and because $\cK$ is a random algorithm, it is run multiple times on each instance. The average across both instances and runs of $\cK$ is reported. One trend of particular importance is that as the level of noise increases, the expected exact recovery rate for either of the alternative candidate heuristics goes to zero.

A natural follow-up question is whether or not, despite the {\it expected} recovery rate going to zero as $\gamma$ increases, if we run $\cK$ many times using using $\Mb=\bSigmaH,~\bSigmaH-\bGammaH$ and select the best clustering found, can we do just as well as FORCE? If yes, then running FORCE (and indeed solving the P-W SDP relaxation in general) offers little benefit over running $\cK$ many times and then attempting to certify the best clustering found.  To answer this question, we denote by $\cK\cB(\Mb,K,N)$ the algorithm which runs $\cK$, defined above, $N$ times and returns the best clustering found in terms of SDP objective value. We compare this to the output of FORCE and choose $N$ to be the maximum of 100 and the number of times FORCE calls a clustering algorithm as a sub-routine on that problem instance. The results in terms of $\EE[d_i(\cK\cB(\Mb,K,N), G^*)]$ are plotted versus gamma in row two of Figure \ref{fig:force_vs_heuristics}; as before, 50 random instances were generated for each level of $\gamma$. Examining the plots we can conclude that solving the SDP not only improves the percentage of points clustered correctly on average, but that it is essential to achieving exact recovery. Using heuristic methods alone cannot achieve the same performance as FORCE or other algorithms that leverage the P-W SDP relaxation.

\begin{figure}[h]
\centering
\hspace*{-1em}
\begin{tabular}{l l}
\includegraphics[width=0.5\textwidth]{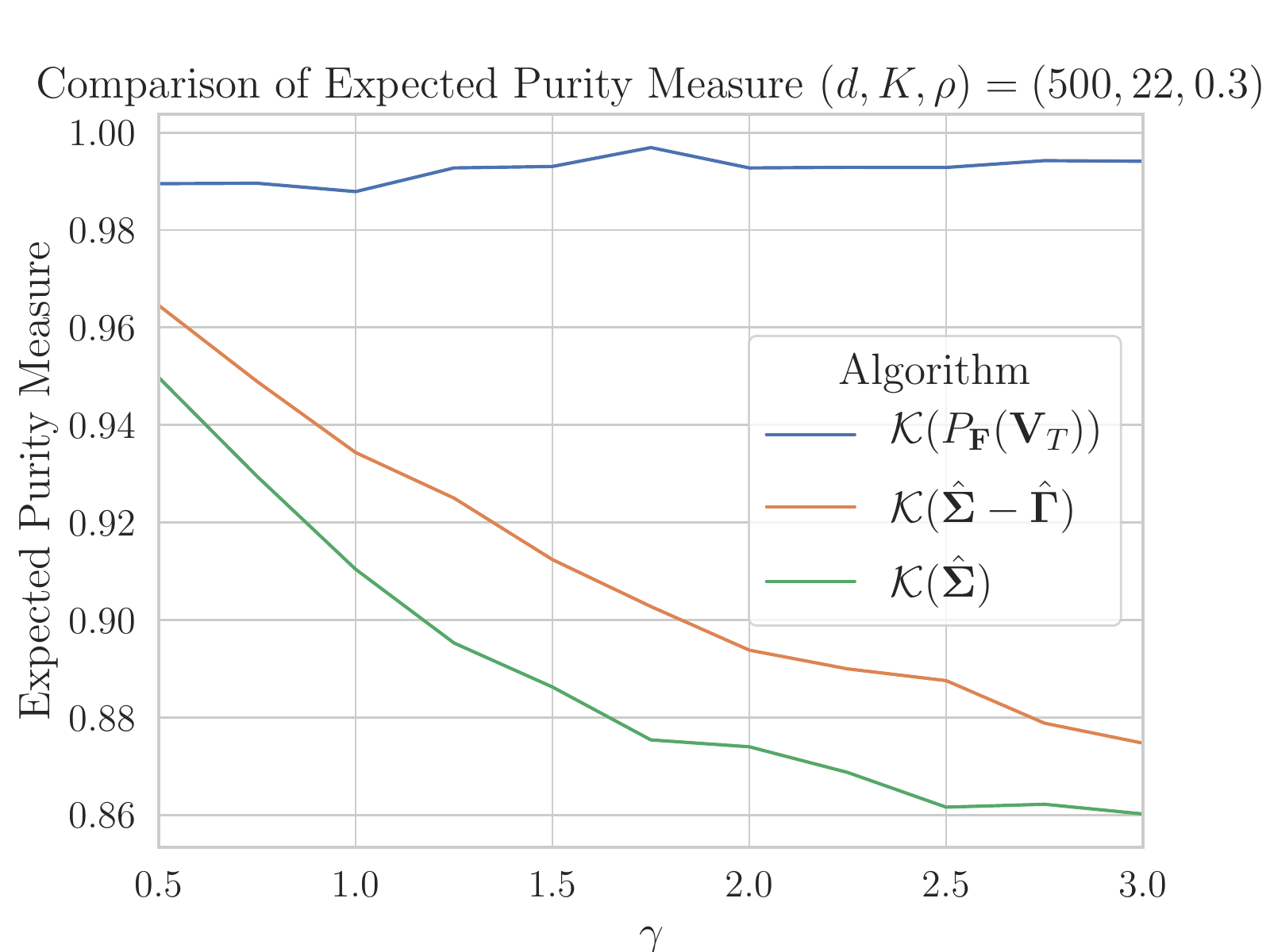} & \hspace*{-1em}\includegraphics[width=0.5\textwidth]{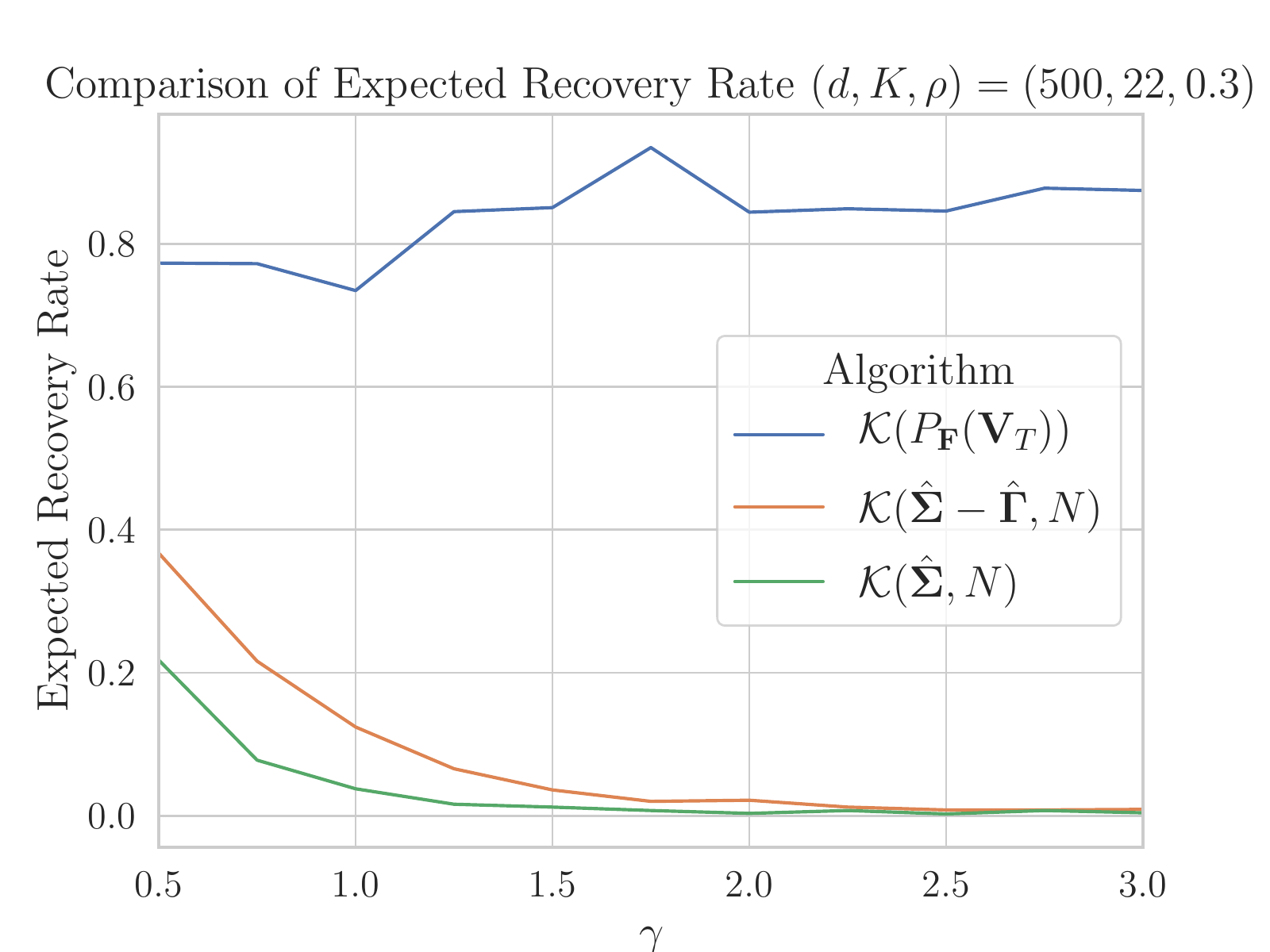} \\
\includegraphics[width=0.5\textwidth]{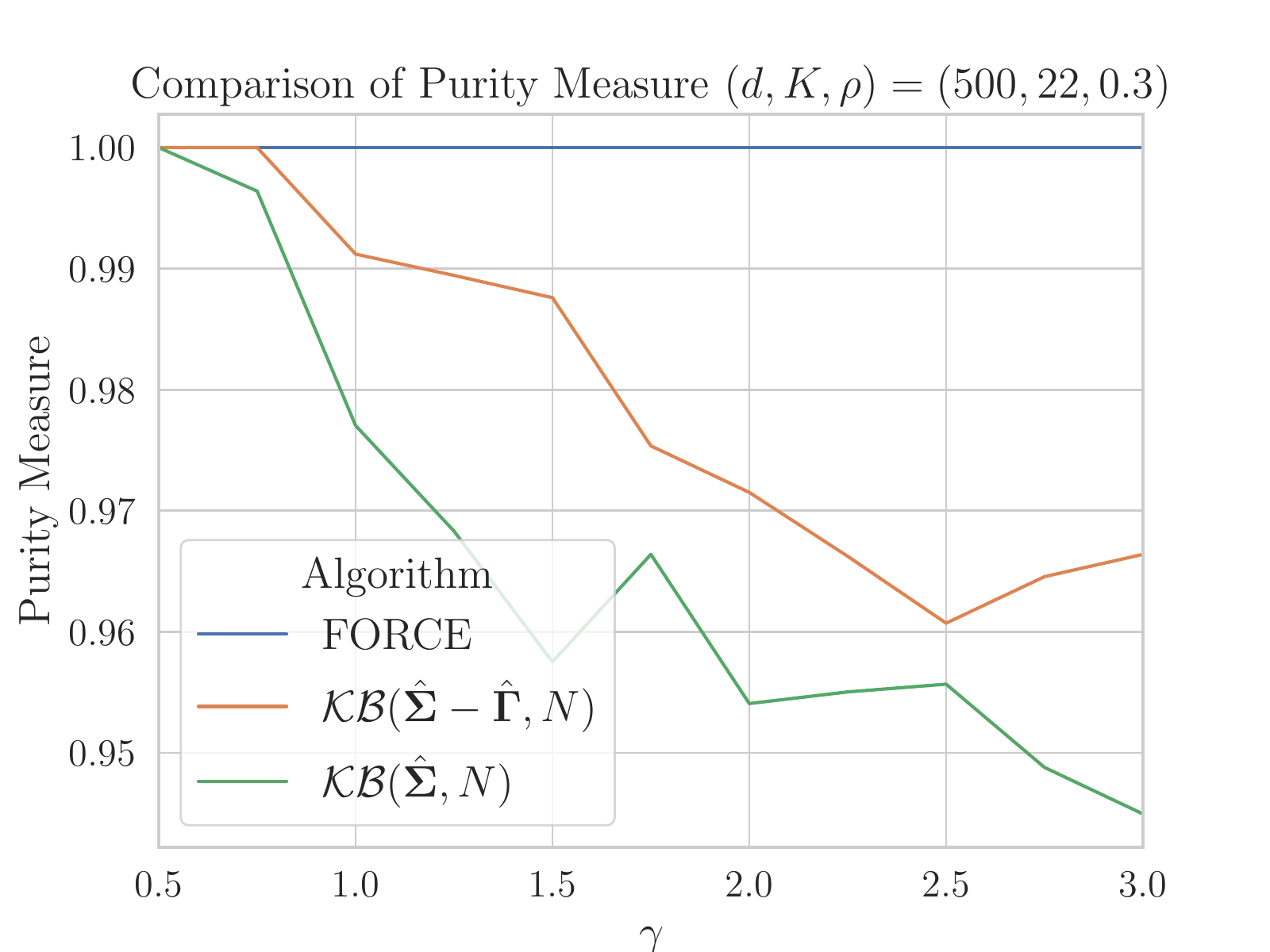} & \hspace*{-1em}\includegraphics[width=0.5\textwidth]{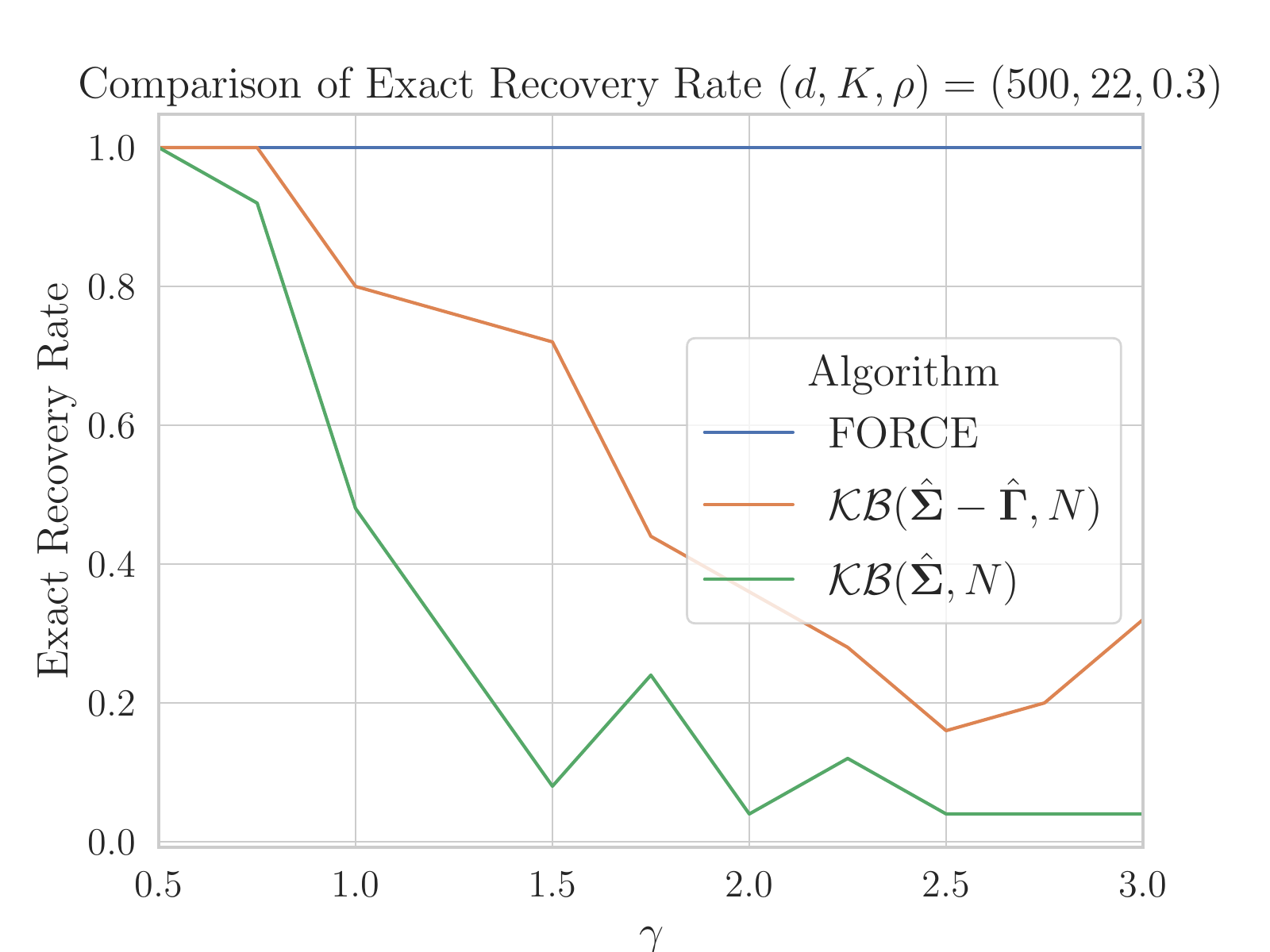} 
\end{tabular}
\caption{Comparison of FORCE with heuristic methods, demonstrating that as noise increases heuristics alone cannot provide high quality clusterings.}
\label{fig:force_vs_heuristics}
\end{figure}

\section{Conclusion}
Motivated by the variable clustering problem, we proposed a new algorithm, FORCE, to solve the P-W SDP which has strong statistical properties in many clustering regimes. FORCE consists of a primal first-order method based on Renegar's method \citep{Renegar2014} and a novel dual certificate construction. We show that for $G$-Latent models satisfying a minimal cluster separation condition, FORCE is guaranteed with high probability to both recover the true latent structure $G^*$ and provide a certificate of having done so. We extended our results to a variant of the P-W SDP where $K$ is not known a priori.

One interesting consequence of our certificate existence theorems, Theorems \ref{thm:gblock_existence} and \ref{thm:gblock_adapt_existence}, is that they show for $G$-Latent models, the SDPs \eqref{eqn:kmeans_sdp} and \eqref{eqn:kmeans_adapt_sdp} are tight with high probability for $\Delta(C^*)$ sufficiently large. Indeed we recover nearly the same minimal cluster separation rate as \citet{Bunea2016}, with the distinction that our proof is constructive in nature.

Our numerical studies clearly indicate the success of FORCE in the variable clustering setting. In our simulation studies, only one other method, ADMM, was able to scale to high dimensions, and it often did not converge in high noise designs. Our studies also verified that solving the P-W SDP was essential to achieve high quality clusterings as noise increased (see Figure \ref{fig:force_vs_heuristics}).  In future work it would be of interest to study the properties of the FORCE dual certificate under other generating distributions for variable and data clustering. The FORCE algorithm is available in the R package GFORCE on CRAN.

\bibliographystyle{ims_nourl_eprint}
\bibliography{/Shared_Documents/Bibliography/library}

\appendix

\section{Proofs Omitted in Section \ref{sec:FORCE_dual}}
First we have a lemma regarding the concentration of the noise terms $\Eb$ about their mean. Sometimes rather than state these concentration results in terms of $d$, we state them in terms of $t \geq d$ to allow for more precise control of constants in our main theorems.  We let $\cE$ denote the event that $||\hat\bGamma - \bGammaS ||_{\infty} \leq p_1||\bGammaS||_{\max}\sqrt{\frac{\log d}{n}}$.
\begin{lemma}
\label{lem:conc_ei_hw}
Under the notation and assumptions from previous sections, if $t\geq d$ then
\[
\biggl|\sum_{j=1}^n \bone^T\bE_{\Gs{i}}^j\bE_{\Gs{i}}^{jT}\bone - \bone^T\bGamma^*_{\Gs{i},\Gs{i}}\bone\biggl| \leq c_0 ||\bGammaS||_{\infty}\sqrt{|\Gs{i}|^2 n \log t},
\]
with probability at least $1-\frac{2}{t}$, where $c_0 = c'(1+\sqrt{p_0})$ is a constant that depends only on $p_0$ and the absolute constant $c'$ from Proposition \ref{lem:hanson_wright_gaussian}. Similarly with probability at least $1-\frac{2}{t}$, for $a \in \Gs{i}$,
\[
\biggl|\sum_{j=1}^n \bone^T\bE_{\Gs{i}}^j E^j_{a} - \gamma^*_a \biggr| \leq c_0 ||\bGammaS||_{\infty}\sqrt{|\Gs{i}| n \log t} ,
\]
\end{lemma}
\begin{proof}
To obtain the result, we observe that
\[
\sum_{j=1}^n \bone^T\bE_{\Gs{i}}^j\bE_{\Gs{i}}^{jT}\bone - \bone^T\bGamma^*_{\Gs{i},\Gs{i}}\bone
\]
is a quadratic form of a $n|\Gs{i}|$-dimensional Gaussian random vector with independent entries. In particular, if we define $\Mb$ to be block diagonal with the $i^{th}$ $n\times n$ diagonal block as $(\bGammaS_{\Gs{i},\Gs{i}})^{1/2}\bone \bone^T(\bGammaS_{\Gs{i},\Gs{i}})^{1/2}$, then we can apply Corollary \ref{lem:hanson_wright_gaussian2} with matrix $\Mb$. Because  $||\Mb||_2 \leq ||\bGammaS||_{\infty}|\Gs{i}|$ and $||\Mb||_F \leq ||\bGammaS||_{\infty}|\Gs{i}|\sqrt{n}$, applying the corollary gives
\[
\biggl|\sum_{j=1}^n \bone^T\bE_{\Gs{i}}^j\bE_{\Gs{i}}^{jT}\bone - \bone^T\bGamma^*_{\Gs{i},\Gs{i}}\bone\biggr| \leq c' ||\bGammaS||_{\infty}\left(\sqrt{|\Gs{i}|^2 n \log t} + |\Gs{i}|\log t \right),
\]
with probability at least $1-\frac{2}{t}$. Using the assumption $\log d \leq p_0 n$ gives the desired result. The proof of the second statement follows similarly, taking instead the diagonal blocks of $\Mb$ as $(\bGammaS_{\Gs{i},\Gs{i}})^{1/2}\bone \eb_a^T(\bGammaS_{\Gs{i},\Gs{i}})^{1/2}$, giving $||\Mb||_2 \leq ||\bGammaS||_{\infty}\sqrt{|\Gs{i}|}$ and $||\Mb||_F \leq ||\bGammaS||_{\infty}\sqrt{n|\Gs{i}|}$.
\end{proof}

\subsubsection*{Proof of Lemma \ref{lem:dg_eigenvalue_concentration}}
{\bf Step 1:} 
For notation, $c_i$ will be used to denote absolute constants. The first step is to decompose $\Qb_{i}^{\perp}(\bX)$. Recall that under the G-Latent model, $\Db = -\hat\bSigma + \hat\bGamma$. Substituting that into the expression for $\Qb_{i}^{\perp}(\bX)$ gives
\begin{align*}
\Qb_{i}^{\perp}(\bX) &= \underlabel{-\frac{1}{|\Gs{i}|^2} \left(\bone^T\hat\bSigma_{\Gs{i},\Gs{i}}\bone \right)\bone\bone^T + \frac{1}{|\Gs{i}|}\left(\bone\bone^T \hat\bSigma_{\Gs{i},\Gs{i}} + \hat\bSigma_{\Gs{i},\Gs{i}}\bone\bone^T\right) - \hat\bSigma_{\Gs{i},\Gs{i}}}{(i)}\\
&\quad + \underlabel{\frac{1}{|\Gs{i}|^2} \left(\bone^T\hat\bGamma_{\Gs{i},\Gs{i}}\bone \right)\bone\bone^T - \frac{1}{|\Gs{i}|}\left(\bone\bone^T \hat\bGamma_{\Gs{i},\Gs{i}} + \hat\bGamma_{\Gs{i},\Gs{i}}\bone\bone^T\right) + \hat\bGamma_{\Gs{i},\Gs{i}}}{(ii)}.
\end{align*}
For (i), we recall that by the definition of the G-Latent model that
\[
\hat\bSigma_{\Gs{i},\Gs{i}} = \frac{1}{n}\sum_{j=1}^n\bX_{\Gs{i}}^j\bX_{\Gs{i}}^{jT} = \sum_{j=1}^n(Z_i^j + \bE_{\Gs{i}}^j)(Z_i^j + \bE_{\Gs{i}}^j)^T.
\]
Plugging this into (i) and simplifying gives us that
\[
\text{(i)} = \frac{1}{n}\sum_{j=1}^n\left(-\frac{\bone^T\bE_{\Gs{i}}^j\bE_{\Gs{i}}^{jT}\bone}{|\Gs{i}|^2} \bone\bone^T + \frac{\bone^T\bE_{\Gs{i}}^j}{|\Gs{i}|}\left(\bone\bE_{\Gs{i}}^{jT} + \bE_{\Gs{i}}^j\bone^T\right) - \bE_{\Gs{i}}^j\bE_{\Gs{i}}^{jT} \right).
\]
Now we see that, again, the expression for $\Qb_{i}^{\perp}(\bX)$ has eight terms. We first show that each concentrates to its mean at the desired rate, and then use the triangle inequality to obtain the final result. Fortunately, we can subtract the mean for each of the 8 terms to the expression for $\Qb_{i}^{\perp}(\bX)$ as the means for (i) are offset by the means for (ii). To give the new decomposition of $\Qb_{i}^{\perp}(\bX)$ explicitly, 
\begin{align*}
\Qb_{i}^{\perp}(\bX) &= -\underlabel{\sum_{j=1}^n\frac{\bone^T\bE_{\Gs{i}}^j\bE_{\Gs{i}}^{jT}\bone}{n|\Gs{i}|^2} \bone\bone^T}{(i).a} + \underlabel{\sum_{j=1}^n\frac{\bone^T\bE_{\Gs{i}}^j}{n|\Gs{i}|}\bone\bE_{\Gs{i}}^{jT}}{(i).b} + \underlabel{\sum_{j=1}^n\frac{\bone^T\bE_{\Gs{i}}^j}{n|\Gs{i}|}\bE_{\Gs{i}}^j\bone^T}{(i).c} - \underlabel{\frac{1}{n}\sum_{j=1}^n\bE_{\Gs{i}}^j\bE_{\Gs{i}}^{jT}}{(i).d}\\
&\quad + \underlabel{\frac{1}{|\Gs{i}|^2} \left(\bone^T\hat\bGamma_{\Gs{i},\Gs{i}}\bone \right)\bone\bone^T}{(ii).a} - \underlabel{\frac{1}{|\Gs{i}|}\bone\bone^T \hat\bGamma_{\Gs{i},\Gs{i}}}{(ii).b} + \underlabel{\frac{1}{|\Gs{i}|}\hat\bGamma_{\Gs{i},\Gs{i}}\bone\bone^T}{(ii).c} + \underlabel{\hat\bGamma_{\Gs{i},\Gs{i}}}{(ii).d}. \numberthis \label{eqn:q_decomp}
\end{align*}
\\\noindent{\bf Step 2:}
For the term (i).a, we can directly apply Lemma \ref{lem:conc_ei_hw}. Doing so, it follows immediately that with probability at least $1-\frac{2}{t}$
\[
\bignorm{\sum_{j=1}^n\frac{\bone^T\bE_{\Gs{i}}^j\bE_{\Gs{i}}^{jT}\bone}{n|\Gs{i}|^2} \bone\bone^T -  \frac{1}{|\Gs{i}|^2}\left(\bone^T\bGamma^*_{\Gs{i},\Gs{i}}\bone \right)\bone\bone^T}_2 \leq c_0 ||\bGamma^*||_{\infty}\sqrt{\frac{\log t}{n}}.
\]
For the term (i).c (and so by symmetry (i).b), we observe that has the form $\ub \vb^T$ and that $||\ub\vb^T||_2 = ||\ub||_2||\vb||_2$. Therefore, we can apply Lemma \ref{lem:conc_ei_hw} and obtain that with probability at least $1-2|\Gs{i}|/t^2$,
\[
\bignorm{\sum_{j=1}^n\frac{\bone^T\bE_{\Gs{i}}^j}{n|\Gs{i}|}\bE_{\Gs{i}}^j\bone^T - \frac{1}{|\Gs{i}|}\bone\bone^T \bGamma^*_{\Gs{i},\Gs{i}}}_2 \leq c_0 ||\bGamma^*||_{\infty} \sqrt{\frac{2\log t}{n}}.
\]
\\\noindent{\bf Step 3:}
Now we control the term (i).d, the sample covariance matrix of the errors. We can directly apply Corollary \ref{cor:conc_covariance} to obtain that with probability at least $1-2/t$
\begin{align*}
\bignorm{\frac{1}{n}\sum_{j=1}^n\bE_{\Gs{i}}^j\bE_{\Gs{i}}^{jT} - \bGamma^*_{\Gs{i},\Gs{i}}}_2 &\leq ||\bGammaS||_{\infty}\left(\frac{|\Gs{i}|}{n} + 2\frac{\sqrt{2|\Gs{i}|\log t}}{n} + 2\sqrt{\frac{|\Gs{i}|}{n}} + (2+\sqrt{p_0})\sqrt{\frac{2\log t}{n}}\right) \\
&\leq ||\bGammaS||_{\infty}\left(\frac{d}{n} + (2+2\sqrt{2p_0})\sqrt{\frac{d}{n}} + (2+\sqrt{p_0})\sqrt{\frac{2\log t}{n}}\right).
\end{align*}
\\\noindent{\bf Step 4:}
For the terms in (ii), consider first (ii).a. We see that
\[
\bignorm{\left(\bone^T\hat\bGamma_{\Gs{i},\Gs{i}}\bone \right)\bone\bone^T - \left(\bone^T\bGamma^*_{\Gs{i},\Gs{i}}\bone \right)\bone\bone^T}_{\max} \leq |\Gs{i}|||\hat\bGamma_{\Gs{i},\Gs{i}} - \bGamma^*_{\Gs{i},\Gs{i}} ||_{\infty}
\]
Conditional on event $\cE$,
\[
\bignorm{\frac{1}{|\Gs{i}|^2} \left(\bone^T\hat\bGamma_{\Gs{i},\Gs{i}}\bone \right)\bone\bone^T - \frac{1}{|\Gs{i}|^2} \left(\bone^T\bGamma^*_{\Gs{i},\Gs{i}}\bone \right)\bone\bone^T}_{\max} \leq \frac{p_1||\bGamma^*||_{\infty}}{|\Gs{i}|}\sqrt{\frac{\log d}{n}}.
\]
Because the matrices above are a multiple of $\bone\bone^T$, it follows that 
\[
\bignorm{\frac{1}{|\Gs{i}|^2} \left(\bone^T\hat\bGamma_{\Gs{i},\Gs{i}}\bone \right)\bone\bone^T - \frac{1}{|\Gs{i}|^2} \left(\bone^T\bGamma^*_{\Gs{i},\Gs{i}}\bone \right)\bone\bone^T}_{2} \leq p_1 ||\bGamma^*||_{\infty}\sqrt{\frac{\log d}{n}}.
\]
Next for (ii).b (and (ii).c by symmetry), we can see that
\begin{align*}
\bignorm{\frac{1}{|\Gs{i}|}\bone\bone^T \hat\bGamma_{\Gs{i},\Gs{i}} - \frac{1}{|\Gs{i}|}\bone\bone^T \bGamma^*_{\Gs{i},\Gs{i}}}_{2} = \frac{1}{|\Gs{i}|}\bignorm{\bone\bone^T \left(\hat\bGamma_{\Gs{i},\Gs{i}} -  \bGamma^*_{\Gs{i},\Gs{i}} \right)}_{2}. \numberthis \label{eqn:q_conc_iib}
\end{align*}
Because $\bGammaH$ and $\bGammaS$ are diagonal, we can use event $\cE$ and the fact that for matrices of the form $\ub \vb^T$,  $||\ub\vb^T||_2 = ||\ub||_2||\vb||_2$, to obtain
\[
\bignorm{\frac{1}{|\Gs{i}|}\bone\bone^T \hat\bGamma_{\Gs{i},\Gs{i}} - \frac{1}{|\Gs{i}|}\bone\bone^T \bGamma^*_{\Gs{i},\Gs{i}}}_{2} \leq p_1 |\bGamma^*|_{\infty}\sqrt{\frac{\log d}{n}}
\]
The same result is immediate for (ii).a by \eqref{eqn:conc_pecok_gamma}. Therefore by combining the above, applying the triangle inequality to \eqref{eqn:q_decomp}, using that $\cE$ occurs with probability at least $1-p_2/d^2$, and choosing $t=d^2$, we find that with probability at least $1-\frac{c_2}{d^2}$ 
\begin{equation*}
||\Qb_{i}^{\perp}(\bX)||_2 \leq c_1||\bGammaS||_{\infty}\left(\frac{d}{n} +\sqrt{\frac{d}{n}}  + \sqrt{\frac{\log d}{n}}  \right),
\end{equation*}
concluding the proof.

\subsubsection*{Proof of Lemma \ref{lem:gblock_yab}}
Under the G-Latent model,
\[
y'_{a,b}(\Xb,y_T) = -\underlabel{\SigmaH{a}{b}}{(i)} + \underlabel{y_a(\Xb,y_T)}{(ii)} + \underlabel{y_b(\Xb,y_T)}{(iii)}
\]
Above, we saw that
\[
y_a(\bX,y_T) = \frac{1}{2|\Gs{i}|^2}\bone^T\Db_{\Gs{i},\Gs{i}}\bone - \frac{1}{|\Gs{i}|}\Db_{a,\Gs{i}}\bone - \frac{1}{2|\Gs{i}|}y_T,
\]
and likewise for $y_b$. Below we denote by $\sigma_1 = \max_i \CS{i}{i}$ and $\sigma_2 = \max\{\max_i \CS{i}{i},||\bGammaS||_{\infty}\}$. Following the same decomposition as in Lemma \ref{lem:dg_eigenvalue_concentration}, we get that
\begin{align*}
y_a(\bX,y_T) &= -\frac{1}{2|\Gs{i}|^2}\bone^T\bSigmaH[\Gs{i},\Gs{i}]\bone + \frac{1}{2|\Gs{i}|^2}\bone^T\bGammaH_{\Gs{i},\Gs{i}}\bone + \frac{1}{|\Gs{i}|}\bSigmaH[a,\Gs{i}] \bone - \GammaH{a}{a} - \frac{1}{2|\Gs{i}|}y_T \\
&= \underlabel{\frac{1}{n}\sum_{l=1}^n \frac{1}{2}(Z_i^l)^2}{(ii).a}  - \underlabel{\frac{1}{2n|\Gs{i}|^2}\sum_{l=1}^n (\bone^T\bE_{\Gs{i}}^l)^2}{(ii).b} +  \underlabel{\frac{1}{n|\Gs{i}|}\sum_{l=1}^n E_a^l \bone^T\bE_{\Gs{i}}^l}{(ii).c} + \underlabel{\frac{1}{n}\sum_{l=1}^n E^l_a Z_i^l}{(ii).d}\\
&\quad + \underlabel{\frac{1}{2|\Gs{i}|^2}\bone^T\bGammaH_{\Gs{i},\Gs{i}}\bone}{(ii).e} - \underlabel{\frac{1}{|\Gs{i}|}\GammaH{a}{a}}{(ii).f} - \frac{1}{2|\Gs{i}|}y_T.
\end{align*}
As in the proof of Lemma \ref{lem:dg_eigenvalue_concentration}, the means of (ii).b and (ii).c offset the means of (ii).e and (ii).f. To control terms (ii).b and (ii).c, by Lemma \ref{lem:conc_ei_hw} with probability at least $1-1/t$,
\[
\frac{1}{2n|\Gs{i}|^2}\sum_{j=1}^n\left(\bone^T\bE_{\Gs{i}}^j\bE_{\Gs{i}}^{jT}\bone - \bone^T\bGamma^*_{\Gs{i},\Gs{i}}\bone\right) \leq \frac{c_0 ||\bGamma^*||_{\infty}}{2} \sqrt{\frac{\log t}{n |\Gs{i}|^2 }}.
\]
Likewise, by Lemma \ref{lem:conc_ei_hw},
\[
\frac{1}{n|\Gs{i}|}\sum_{i=1}^n \left(E_a \bE_{\Gs{i}}^{jT}\bone - \gammaS{a}\right) \geq -c_0 ||\bGamma^*||_{\infty}\sqrt{\frac{\log t}{n |\Gs{i}| }},
\]
with probability at least $1-1/t$. Conditional on event $\cE$, \eqref{eqn:conc_pecok_gamma} shows that
\begin{align*}
\frac{1}{2|\Gs{i}|^2}\left(\bone^T\bGammaH_{\Gs{i},\Gs{i}}\bone - \bone^T\bGammaS_{\Gs{i},\Gs{i}}\bone \right) &\geq -p_1 ||\bGamma^*||_{\infty}\sqrt{\frac{\log d}{n |\Gs{i}| }},\\
\frac{1}{|\Gs{i}|}\left(\GammaH{a}{a} - \Gamma^*_{a,a} \right) &\leq p_1||\bGamma^*||_{\infty}\sqrt{\frac{\log d}{n |\Gs{i}| }}.
\end{align*}
Lastly, if we denote by $\sigma_1= \max_i \CS{i}{i}$, term (ii).d can be bounded by using Corollary \ref{lem:hanson_wright_gaussian2}, which gives that
\begin{equation}
\label{eqn:yab_ii_1}
\frac{1}{n}\sum_{l=1}^n E^l_a Z_i^l \geq - c_0||\bGammaS||_{\infty}^{1/2}\sigma_1^{1/2} \sqrt{\frac{\log t}{n}},
\end{equation}
with probability at least $1-1/t$. The same results can be obtained for $y_b$. For the terms in (i), we expand as before:
\[
\SigmaH{a,b} = \underlabel{\frac{1}{n}\sum_{l=1}^n Z_i^lZ_j^l}{(i).a} + \underlabel{\frac{1}{n}\sum_{l=1}^nE_a^lZ_j^l}{(i).b} + \underlabel{\frac{1}{n}\sum_{l=1}^nE_b^lZ_i^l}{(i).c} + \underlabel{\frac{1}{n}\sum_{l=1}^n E_a^l E_b^l}{(i).d}.
\]
Terms (i).b and (i).c can be bounded in the same way as \eqref{eqn:yab_ii_1}. Term (i).d can be bounded by Corollary \ref{lem:hanson_wright_gaussian2}, giving that
\[
\frac{1}{n}\sum_{l=1}^n E_a^l E_b^l \geq -c_0||\bGamma^*||_{\infty}\sqrt{\frac{\log t}{n}},
\]
with probability at least $1-1/t$. All that remains is to bound the terms (i).a, (ii).a and (iii).a. Fortunately, these correspond to the population quantity $\Delta \bCS$. Observing that this is just a quadratic form of $2n$-dimensional Gaussian vector, we can applying Lemma \ref{lem:conc_ei_hw}. Doing so gives that
\[
\frac{1}{2n}\left(\sum_{l=1}^n (Z_i^l)^2  + \sum_{l=1}^n (Z_j^l)^2 - 2\sum_{l=1}^n Z_i^lZ_j^l \right) \geq \frac{1}{2}\left(\CS{i}{i} + \CS{j}{j} - \CS{i}{j}\right) - 2c_0 \sigma_1\sqrt{\frac{\log t}{n}}
\]
with probability at least $1-1/t$. Combining all the bounds for (i)-(iii), using that $\cE$ occurs with probability at least $1-p_2/d^3$, and selecting $t=d^3$, we can see that, with probability at least $1-c_1/d^3$
\begin{align*}
y'_{a,b} &\geq \frac{1}{2}(\CS{i}{i} + \CS{j}{j} - 2\CS{i}{j}) - \frac{1}{2|\Gs{i}|}y_T - \frac{1}{2|\Gs{j}|}y_T - c_1||\bGammaS||_{\infty}\sqrt{\frac{\log d}{n |\Gs{i}| }} - c_2\sigma\sqrt{\frac{\log d}{n}}\\
&\geq \frac{1}{2}\Delta(\bCS) - \frac{1}{2|\Gs{i}|}y_T - \frac{1}{2|\Gs{j}|}y_T - c_1||\bGammaS||_{\infty}\sqrt{\frac{\log d}{n |\Gs{i}| }} - c_2\sigma\sqrt{\frac{\log d}{n}}.
\end{align*}

\section{Some Technical Lemmas}
\label{sec:tech_lemmas}
\begin{lemma}
\label{lem:F_props}
Let $\Mb$ be a $d\times d$ real, symmetric matrix of the form
\[
\Mb = a\Ib + b\bone\bone^T.
\]
where $a,b \in \RR$ then $\Mb$ has eigenvalues $a+b$ with multiplicity 1 and $a$ with multiplicity $d-1$. If $a,b > 0$, then $\Mb$ also has the  property that
\begin{align*}
\Mb^{1/2} =& \sqrt{a}\Ib + \frac{\sqrt{a+db} - \sqrt{a}}{d}\bone\bone^T,\\
\Mb^{-1} =& \frac{1}{a}\Ib - \frac{b}{a^2+abd}\bone\bone^T,\\
\Mb^{-1/2} =& \frac{1}{\sqrt{a}}\Ib - \frac{\sqrt{a+db} - \sqrt{a}}{d\sqrt{a^2 + dab}}\bone\bone^T.
\end{align*}
\end{lemma}
\begin{proof}[Proof of Lemma \ref{lem:F_props}]
Using the Sherman-Morrison formula, a matrix of the form $\Mb = a\Ib + b\bone\bone^T$, where $a,b > 0$ has the inverse
\[
\Mb^{-1} = \frac{1}{a}\Ib - \frac{b}{a^2+abd}\bone\bone^T.
\]
Because $\Mb \succ 0$, all eigenvalues are strictly positive and denote by $\lambda_i$ and $q_i$ the eigenvalues and corresponding eigenvectors. Without loss of generality, let $q_i$ be orthonormal. Then we can write $\Mb = \sum_i \lambda_i \qb_i\qb_i^T$. By the form of $\Mb$, clearly $\frac{1}{\sqrt{d}}\bone$ is always an eigenvector of $\Mb$ with eigenvalue $a+db$, so we can take $q_1 = \frac{1}{\sqrt{d}}\bone$ and $\lambda_1 = 1$.
The remaining $q_i$ span $(\bone\bone^T)^{\perp}$ and have corresponding eigenvalues $\lambda_i = a$. Therefore, 
\[
\Mb^{1/2} = \frac{\sqrt{a+db}}{\sqrt{d}}\bone\bone^T + \sum_{i=2}^d \sqrt{a}\qb_i\qb_i^T.
\]
Because this eigen-decomposition is unique, the above gives
\[
\Mb^{1/2} = \sqrt{a}\Ib + \frac{\sqrt{a+db} - \sqrt{a}}{d}\bone\bone^T.
\]
Using the expression for $\Mb^{-1}$ given above, it follows that
\[
\Mb^{-1/2} = \frac{1}{\sqrt{a}}\Ib - \frac{\sqrt{a+db} - \sqrt{a}}{d\sqrt{a^2 + dab}}\bone\bone^T.
\]
\end{proof}

The following result for quadratic forms of standard multivariate Gaussian random variables can be found in many forms in the literature (for example, \citet{Rudelson2013}).
\begin{lemma}[Hanson-Wright Inequality for Gaussian Random Variables]
\label{lem:hanson_wright_gaussian}
Let $\bX \sim \cN(0,\Ib)$ be a $d$-dimensional random vector and let $\Ab$ be a $d \times d$ matrix in $\RR^{d \times d}$. Then
\[
\PP\left(|\bX^T \Ab \bX - \EE\left[\bX^T \Ab \bX \right] | \geq t \right) \leq 2 \exp\left(-c \min\left\{\frac{t^2}{||\Ab||_F^2},\frac{t}{||\Ab||_2} \right\} \right),
\]
for some absolute constant $c$.
\end{lemma}
In particular, the following corollary is useful.
\begin{corollary}
\label{lem:hanson_wright_gaussian2}
Let $\bX \sim \cN(0,\Ib)$ be a $d$-dimensional random vector and let $\Ab$ be a $d \times d$ matrix in $\RR^{d \times d}$. Then
\[
\PP\left(|\bX^T \Ab \bX - \EE\left[\bX^T \Ab \bX \right] | \geq ||\Ab||_F\sqrt{t} + ||\Ab||_2 t \right) \leq 2 \exp\left(-c t \right),
\]
for some absolute constant $c$. Equivalently, 
\[
|\bX^T \Ab \bX - \EE\left[\bX^T \Ab \bX \right]| \leq c'\left(||\Ab||_F\sqrt{\log t} + ||\Ab||_2 \log t \right)
\]
with probability at least $1-2/t$ for some absolute constant $c'$.
\end{corollary}

Below we are concerned with the rate of concentration in the spectral norm of a sample covariance matrix to its mean: $||\hat\bSigma - \bSigma^*||_2$.  If we write $\hat\bSigma = \frac{1}{n}\bX^T\bX$, where $\bX$ refers to the $n\times d$ matrix in which the rows are the observations $\bX_i$, we see how such a result is directly applicable to the problem at hand. We repeat the statement of Gordon's Theorem given in \citet{Vershynin2011} below as Proposition \ref{prop:gordon}. We use the notation from \citet{Vershynin2011} of $s_{\min}$ and $s_{\max}$ to denote the smallest and largest singular values, respectively.
\begin{proposition}
\label{prop:gordon}
Let $\bX$ be an $n \times d$ matrix whose entries are independent standard normal random variables. Then
\[
\sqrt{n} - \sqrt{d} \leq \EE[s_{\min}(\bX)] \leq \EE[s_{\max}(\bX)] \leq \sqrt{n} + \sqrt{d}
\]
\end{proposition} 
Using the result on sub-Gaussian concentration of a Lipschitz function of independent random variables, we immediately obtain the following corollary (also given in \citet{Vershynin2011}).
\begin{corollary}
\label{cor:sing_vals}
Let $\bX$ be an $n \times d$ matrix whose entries are independent standard normal random variables, then for every $t \geq 0$
\[
\sqrt{n} - \sqrt{d} -t \leq s_{\min}(\bX) \leq s_{\max}(\bX) \leq \sqrt{n} + \sqrt{d} + t
\]
with probability at least $1-2\exp(-t^2/2)$.
\end{corollary}
\begin{proof}
Observing that the functions $s_{\min}$ and $s_{\max}$ are 1-Lipschitz and using the sub-Gaussian tail bound, the result is immediate from the above.
\end{proof}

\begin{corollary}
\label{cor:conc_covariance}
Let $\bX_i$, for $i=1,\dots,n$, be a $d$-dimensional random vector sampled from $\cN(0,\bSigma)$. Denoting $\bSigmaH := n^{-1}\sum_{i=1}^n\bX_i\bX_i^\top$, we have that 
\begin{align*}
\lambda_{\min}\left(\bSigmaH-\bSigma\right) &\geq \lambda_{\min}(\bSigma)\left(\frac{d}{n} + \frac{2t\sqrt{d}}{n} + \frac{t^2}{n} - \frac{2(\sqrt{d} + t)}{\sqrt{n}} \right),\\
\lambda_{\max}\left(\bSigmaH-\bSigma\right) &\leq \lambda_{\max}(\bSigma)\left(\frac{d}{n} + \frac{2t\sqrt{d}}{n} + \frac{t^2}{n} + \frac{2(\sqrt{d} + t)}{\sqrt{n}} \right),
\end{align*}
with probability at least $1-2\exp(-t^2/2)$.
\end{corollary}
\begin{proof}
This follows directly from Corollary \ref{cor:sing_vals}.
\end{proof}

\section{Extension of First-Order SDP Results}
\label{sec:renegar_method_details}
This section contains the derivations of the convergence rate of the modified Renegar's method used in Section \ref{sec:FORCE_all}. First we mention that one way to avoid the $\Fb \neq \Ib$ issue, as shown in \citet{Renegar2014}, is to instead solve the rotated problem
\begin{equation}
\begin{aligned}
& \underset{\Vb}{\text{maximize}} & & \lmin{\Vb}  \\
& \text{subject to} & & \trip{\Fb^{1/2} \Ab_i \Fb^{1/2}}{\Vb} = b_i \text{ for } i=1,\dots,p \\
&&& \trip{\Fb^{1/2}\Db\Fb^{1/2}}{\Vb} = u_0.
\end{aligned}
\label{eqn:lmin_form_rotated}
\end{equation}
Rotating the system of constraints is not a satisfactory solution for \eqref{eqn:kmeans_sdp} because the easy projection onto $\cC_{\lambda}^{\perp}$ is lost.  Thus we need to carefully analyze the smoothness of the objective function $f_{\mu,\Fb}$ yielding similar results as the case when $\Fb = \Ib$.

\subsection{Extension of the Smoothed Scheme to Arbitrary Initial Solutions}
\label{sec:renegar_method_details:smoothed_E}
For completeness, we give in this section the extension of the results in \citet{Renegar2014} to arbitrary choice of initial feasible solution $\Fb$. Similar to the notation in \citet{Renegar2014}, we denote the smoothed approximation of $\lminf{\Vb}$ as
\begin{equation}
f_{\mu,\Fb}(\Vb) = - \mu \log \left(\sum_j \exp\left(-\lambda_{j}(\Fb^{-1/2} \Vb \Fb^{-1/2}) / \mu \right) \right),
\end{equation}
where $\lambda_j$ denotes the $j^{th}$ eigenvalue of $\Vb$.
\begin{lemma}
\label{lem:smoothness_constant}
The function $f_{\mu,\Fb}(\Vb)$ is $\frac{||\Fb^{-1}||^2_2}{\mu}$-smooth.
\end{lemma}
\begin{proof}
From \citet{Nesterov2005} we have that
\[
f_{\mu}(\Vb) = - \mu \log \left(\sum_j \exp\left(-\lambda_{j}(\Vb) / \mu \right) \right)
\]
is $1/\mu$-smooth. Denote by $g : \RR^{d\times d} \rightarrow \RR^{d\times d}$ the mapping $g(\Vb) = \Fb^{-1/2} \Vb \Fb^{-1/2}$. Using differential notation, we see can obtain that
\[
\diff g(\Vb) =  \Fb^{-1/2} \diff \Vb \Fb^{-1/2}.
\]
By Cauchy invariance, and vectorizing $g$, we obtain that the Jacobian is $\deriv \vect g(\Vb) = \Fb^{-1/2}\otimes\Fb^{-1/2}$. To simplify the proof, we now view $f_\mu$ and $f_{\mu,\Fb}$ as functions on $\RR^{d^2}$. By the chain rule for the Jacobian,
\[
\deriv f_{\mu,\Fb}(\Vb) = \deriv f_{\mu}(g(\Vb)) \deriv \vect g(\Vb).
\]
For any $\Vb$ and $\Ub$ in $\RR^{d\times d}$, we obtain
\begin{align*}
||\deriv f_{\mu,\Fb}(\Vb) - \deriv f_{\mu,\Fb}(\Ub)|| &= ||\deriv f_{\mu}(g(\Vb)) \deriv \vect g - \deriv f_{\mu}(g(\Ub)) \deriv \vect g||\\
&\leq ||\deriv \vect g||_2 ||\deriv f_{\mu}(g(\Vb)) - \deriv f_{\mu}(g(\Ub))|| \\
&\leq \frac{||\deriv \vect g||_2}{\mu} ||g(\Vb) - g(\Ub)|| \\
&= \frac{||\deriv \vect g||_2}{\mu} ||\Fb^{-1/2}(\Vb - \Ub)\Fb^{-1/2}||\\
&= \frac{||\deriv \vect g||_2}{\mu} ||\Fb^{-1/2}\otimes\Fb^{-1/2}\vect(\Vb - \Ub)||\\
&\leq \frac{||\Fb^{-1/2}||^4_2}{\mu} ||\Vb - \Ub||,
\end{align*}
proving the result.
\end{proof}

The smoothed form of \eqref{eqn:lmin_form} is
\begin{equation}
\begin{aligned}
& \underset{\Vb}{\text{maximize}} & & f_{\mu,\Fb}(\Vb)  \\
& \text{subject to} & & \trip{\Ab_i}{\Vb} = b_i \text{ for } i=1,\dots,p \\
&&& \trip{\Db}{\Vb} = u_0.
\end{aligned}
\label{eqn:lmin_form_smoothed2}
\end{equation}

The underlying sub-gradient descent method used in \citet{Renegar2014} is from Chapter 3 in \citet{Nesterov2004}, adapted to \eqref{eqn:lmin_form}. The convergence analysis is presented below. We denote the optimal solution to \eqref{eqn:lmin_form} as $\Vb^*_{u_0}$ because the solution is within the level set corresponding to $u_0$ in the original problem.

Theorem \ref{thm:nesterov_acc_pgd} gives the rate for the accelerated projected sub-gradient method, applied to a smooth objective function.  Using Nesterov's acceleration for constrained optimization (Algorithm \ref{alg:nesterov_acc_pgd}) we can adapt the results in Sections 6 and 7 of \citet{Renegar2014} to the more general problem with arbitrary $\Fb$. For \eqref{eqn:lmin_form_smoothed2}, Algorithm \ref{alg:nesterov_acc_pgd_lmine} gives more details of Nesterov's acceleration applied to our problem of interest.

\begin{algorithm}
\caption{Nesterov's Accelerated Projected Gradient Descent for \eqref{eqn:lmin_form_smoothed2}}
\label{alg:nesterov_acc_pgd_lmine}
\begin{algorithmic}
\Input $T$, $\Ub_1 \in \cD$, $\beta$, $\{ \lambda_t \}$ and $\{ \gamma_t \}$
\Output $\Ub_{T}$
\State $\Vb_1 \gets \Ub_1$
\For{$t \gets 1,\dots,T-1$}
    \State $\Ub_{t+1} = \Vb_t + \frac{1}{\beta} \cP_{\cC_{\lambda}^{\perp}}( \nabla f_{\mu,\Fb}(\Vb_t) )$
    \State $\Vb_{t+1} = (1-\gamma_t)\Ub_{t+1} + \gamma_t \Ub_t$
\EndFor
\State \Return $\Ub_{T}$
\end{algorithmic}
\end{algorithm}
In Algorithm \ref{alg:nesterov_acc_pgd_lmine}, $\beta = \frac{||\Fb^{-1}||_2^2}{\mu}$. Notationally, we denote the optimal solution to \eqref{eqn:lmin_form_smoothed2} as $\Vb^*_{u_0}(\mu) $. Theorem \ref{thm:acc_pgd_lmin_rate} gives the convergence rate.
\begin{theorem}[Analogue to 6.1 in \citet{Renegar2014}]
\label{thm:acc_pgd_lmin_rate}
Let $\epsilon' > 0$ and $\mu = \frac{\epsilon'}{2 \log d}$. Applying Algorithm \ref{alg:nesterov_acc_pgd_lmine} with initial iterate $\Ub_1$ satisfying $u_0 = \trip{\Db}{\Ub_1} < \trip{\Db}{\Fb}$ and with
\[
T \geq \frac{2\sqrt{\log d} ||\Fb^{-1}||_2^2  ||\Ub_1 - \Vb^*_{u_0}(\mu)||_F }{\epsilon'}
\]
gives that
\[
\lminf{\Vb^*_{u_0}} - \lminf{\Ub_T} \leq \epsilon'.
\]
\end{theorem}
\begin{proof}[Proof of Theorem \ref{thm:acc_pgd_lmin_rate}]
This follows mainly from \ref{thm:nesterov_acc_pgd} and that
\[
\lminf{\Ub} - \mu \log d \leq f_{\mu,\Fb}(\Ub) \leq \lminf{\Ub}.
\]
\end{proof}
\begin{corollary}[Analogue to 6.2 in \citet{Renegar2014}]
\label{thm:acc_pgd_lmin_rate_2}
Let $\epsilon' > 0$ and $\mu = \frac{\epsilon'}{2 \log d}$. Applying Algorithm \ref{alg:nesterov_acc_pgd_lmine} with initial iterate $\Ub_1$ satisfying $u_0 = \trip{\Db}{\Ub_1} < \trip{\Db}{\Fb}$ and with
\[
T \geq \frac{2\sqrt{\log d} ||\Fb^{-1}||_2^2  R }{\epsilon'}
\]
gives that
\[
\lminf{\Vb^*_{u_0}} - \lminf{\Ub_T} \leq \epsilon',
\]
where
\[
R = \max\{||\Ub- \Vb||_F : \Ub,\Vb \text{ are feasible for \eqref{eqn:sdp_form} and } \trip{\Db}{\Ub} \leq \trip{\Db}{\Fb}, \trip{\Db}{\Vb} \leq \trip{\Db}{\Fb} \}.
\]
\end{corollary}
\begin{proof}[Proof of Corollary \ref{thm:acc_pgd_lmin_rate_2}]
See proof of 6.2 in \citet{Renegar2014}. The proof here is the same. The main idea is $\Vb_{u_0}^*(\mu)$ is feasible for \eqref{eqn:sdp_form}.
\end{proof}
The Corollary above gives a bound on the solution to \eqref{eqn:lmin_form}, but what we want is a bound on the solution to \ref{eqn:kmeans_sdp}. Clearly, however, this depends on the inputs to the algorithm. This is summarized in the next Corollary.

\begin{corollary}[Analogous to 6.3 in \citet{Renegar2014}]
\label{thm:acc_pgd_lmin_rate_3}
Let $\epsilon' > 0$ and $\mu = \frac{\epsilon'}{6 \log d}$. Assume that
\[
\lminf{\Ub_1} \geq \frac{1}{6} \text{ and } \frac{\trip{\Db}{\Fb} - v^* }{\trip{\Db}{\Fb} - v_0 } \leq 3
\]
Applying Algorithm \ref{alg:nesterov_acc_pgd_lmine} with initial iterate $\Ub_1$ satisfying $u_0 = \trip{\Db}{\Ub_1} < \trip{\Db}{\Fb}$ and with
\[
T \geq \frac{2\sqrt{\log d} ||\Fb^{-1}||_2^2  R }{\epsilon}
\]
gives that
\[
\frac{\trip{\Db}{P_{\Fb}(\Ub_T)} - u^* }{\trip{\Db}{\Fb} - u^* } \leq \epsilon,
\]
where
\[
R = \max\{||\Ub- \Vb||_F : \Ub,\Vb \text{ are feasible for \eqref{eqn:sdp_form} and } \trip{\Db}{\Ub} \leq \trip{\Db}{\Fb}, \trip{\Db}{\Vb} \leq \trip{\Db}{\Fb} \}.
\]
\end{corollary}
\begin{proof}[Proof of Corollary \ref{thm:acc_pgd_lmin_rate_3}]
We can apply Corollary \ref{thm:acc_pgd_lmin_rate_2} to get the result.
\end{proof}
From \ref{thm:acc_pgd_lmin_rate_3} it is clear that if we can find an initial iterate satisfying a certain closeness to optimality, then we are closer to an algorithm that does not require knowledge of the optimal value as input. This can be accomplished using Algorithm \ref{alg:smoothed_subscheme} and Algorithm \ref{alg:smoothed_scheme}. Lemma \ref{lem:smoothed_subscheme_rate} establishes the required conditions and gives the rate for Algorithm \ref{alg:smoothed_subscheme}.
\begin{algorithm}
\caption{Smoothed Subscheme for \eqref{eqn:lmin_form_smoothed2} \citep{Renegar2014}}
\label{alg:smoothed_subscheme}
\begin{algorithmic}
\Input $\epsilon$, $\bU_0 \in \cC$ such that $\trip{\Db}{\bU_0} < \trip{\Db}{\Fb}$ and $\lminf{\bU_0} = \frac{1}{6}$
\Output $\bU_L$ such that $\lminf{\bU_L} = \frac{1}{6}$ and $\frac{\trip{\Db}{\Fb} - u^* }{\trip{\Db}{\Fb} - u_L } \leq 3$
\State $l \gets 0$ (Outer Iterations Counter)
\State $\mu \gets \frac{1}{6\log d}$
\State $T \gets 2\sqrt{\log d} ||\Fb^{-1}||_2^2 R$
\State $u_0 = \trip{\Db}{\bU_0}$
\State $\mathrm{done} \gets \mathrm{FALSE}$
\While{$!\mathrm{done}$}
    \State Apply Algorithm \ref{alg:nesterov_acc_pgd_lmine} to \eqref{eqn:lmin_form_smoothed2} on level set corresponding to $u_l$ and inputs $T$, $\bU_l$. Denote the output by $\bV_{l}$.
    \If{ $\lminf{\bU_{l+1}} \leq \frac{1}{3}$ }
        \State $\mathrm{done} \gets \mathrm{TRUE}$
    \Else
        \State $\bU_{l+1} \gets  \Fb + \frac{5}{6}\frac{1}{1-\lminf{\bV_l}}\left(\bV_l - \Fb \right)$
        \State $u_{l+1} = \trip{\Db}{\bU_{l+1}}$
        \State $l \gets l + 1$
    \EndIf
\EndWhile
\State $\bV_L = \bV_l$
\State \Return $\bV_L$
\end{algorithmic}
\end{algorithm}
\begin{algorithm}
\caption{Smoothed Scheme for \eqref{eqn:lmin_form_smoothed2} \citep{Renegar2014}}
\label{alg:smoothed_scheme}
\begin{algorithmic}
\Input $0<\epsilon<1$ and $\bU_0$ such that $\trip{\Db}{\bU_0} < \trip{\Db}{\Fb}$ and $\lminf{\bU_0} = \frac{1}{6}$ and $\bU_0$ feasible for \eqref{eqn:sdp_form}.
\Output $P_{\Fb}(\Vb)$
\State Apply Algorithm \ref{alg:smoothed_subscheme} with input $\bU_0$. Let $\Ub_1$ denote its output.
\State $T \gets \lceil\frac{2\sqrt{\log d} ||\Fb^{-1}||_2^2  R }{\epsilon}\rceil $
\State $\mu \gets \frac{\epsilon}{6 \log d}$
\State Apply Algorithm \ref{alg:nesterov_acc_pgd_lmine} with inputs $T$, $\Ub_1$, $\mu$ on \eqref{eqn:lmin_form_smoothed2} with level set $u_1$. Denote the output by $\Vb$.
\State \Return $P_{\Fb}(\Vb)$
\end{algorithmic}
\end{algorithm}

\begin{lemma}[Analogue to Proposition 7.1 \citet{Renegar2014}]
\label{lem:smoothed_subscheme_rate}
Assuming inputs as stated, Algorithm \ref{alg:smoothed_subscheme} terminates with a matrix $\bU_L$ which is feasible for \eqref{eqn:sdp_form} and satisfies
\[
\lminf{\bU_L} = \frac{1}{6}, \frac{\trip{\Db}{\Fb} - u^*}{\trip{\Db}{\Fb} - \trip{\Db}{\bU_L}} \leq 3.
\]
Furthermore, the number of outer iterations $L$, is bounded by
\[
L \leq \log_{5/4}\left(\frac{\trip{\Db}{\Fb} - u^*}{\trip{\Db}{\Fb} - u_0} \right),
\]
where $u_0 = \trip{\Db}{\bU_0}$.
\end{lemma}
\begin{proof}[Proof of Lemma \ref{lem:smoothed_subscheme_rate}]
See the proof of Proposition 7.1. The rate from \citet{Bubeck2015} can be used in place of that from \citet{Nesterov2004}.
\end{proof}

\begin{theorem}[Analogue to Theorem 7.2 \citet{Renegar2014}]
\label{thm:smoothed_scheme_rate}
Assuming inputs as stated, Algorithm \ref{alg:smoothed_scheme} terminates with a matrix $\Ub$ which is feasible for \eqref{eqn:sdp_form} and satisfies
\[
\frac{\trip{\Db}{\Ub} - u^*}{\trip{\Db}{\Fb} - u^*} \leq \epsilon.
\]
Furthermore, the total number of iterations of Algorithm \ref{alg:nesterov_acc_pgd_lmine} is bounded by
\[
2R||\Fb^{-1}||_2^2\sqrt{\log d}  \left(\frac{1}{\epsilon} + \log_{5/4}\left(\frac{\trip{\Db}{\Fb} - u^*}{\trip{\Db}{\Fb} - u_0} \right)\right),
\]
where $u_0 = \trip{\Db}{\bU_0}$.
\end{theorem}
\begin{proof}[Proof of Theorem \ref{thm:smoothed_scheme_rate}]
Follows from \ref{lem:smoothed_subscheme_rate}.
\end{proof}

\section{Accelerated Projected Gradient Descent}
\label{sec:nesterov_acc_pgd}
In this section we give, for completeness, a proof of Nesterov's acceleration for smooth, constrained optimization problems. The algorithm is summarized as Algorithm \ref{alg:nesterov_acc_pgd}. The problem is phrased as a minimization
\begin{equation}
x \in \argmin_{x \in \cC} f(x)
\end{equation}
for some $\beta$-smooth, convex $f(x)$, Algorithm \ref{alg:nesterov_acc_pgd} gives Nesterov's accelerated projected gradient descent over a convex set $\cC$. Following \citet{Bubeck2015} we can define the auxiliary sequences $\{ \lambda_t \}$ and $\{ \gamma_t \}$.
\begin{equation}
\label{eqn:aux_sequence_acc_pgd}
\lambda_0 = 0 \text{\quad and \quad} \lambda_{t+1} = \frac{1 + \sqrt{1 + 4\lambda_{t}^2}}{2} \text{\quad and \quad} \gamma_t = \frac{1 - \lambda_t}{\lambda_{t+1}}.
\end{equation}
Before the proof, we require Lemma \ref{lem:beta_smoothness}, characterizing $\beta$-smoothness in a way that is helpful.
\begin{lemma}
\label{lem:beta_smoothness}
Consider any $x_t$ and $y$ in a convex set $\cC$. Let $\alpha$ be the gradient update step-size and let $z_{t+1} = \Pi_{\cC} (x_{t+1} - \alpha \nabla f(x_t))$. Then,
\[
f(z_{t+1}) - f(y) \leq g^{\perp}(x_t)^T(x_t - y) - \frac{\alpha}{2}||g^{\perp}(x_t)||_2^2.
\]
\end{lemma}
\begin{proof}
This is a common result, so we omit the proof.
\end{proof} 

\begin{algorithm}
\caption{Nesterov's Accelerated Projected Gradient Descent for $\beta$-smooth $f$}
\label{alg:nesterov_acc_pgd}
\begin{algorithmic}
\Input $T$, $\cC$, $x_1 \in \cC$, $\beta$, $\{ \lambda_t \}$ and $\{ \gamma_t \}$
\Output $z_{T}$
\State $y_1 \gets x_1$
\State $z_1 \gets x_1$
\For{$t \gets 1,\dots,T-1$}
    \State $y_{t+1} \gets x_t - \frac{1}{\beta}\nabla f(x_t) $
    \State $z_{t+1} = \Pi_{\cC} (y_{t+1})$
    \State $x_{t+1} = (1-\gamma_t)z_{t+1} + \gamma_t z_t$
\EndFor
\State \Return $z_{T}$
\end{algorithmic}
\end{algorithm}
\begin{theorem}[Adapted from 3.12 in \citet{Bubeck2015}]
\label{thm:nesterov_acc_pgd}
Let $f$ be a convex, $\beta$-smooth function and $T$ be the number of iterations. Then Algorithm \ref{alg:nesterov_acc_pgd} satisfies
\[
f(z_T) - f(x^*) \leq \frac{2\beta ||x_1 - x^*||^2}{T^2}.
\]
\end{theorem}
\begin{proof}[Proof of Theorem \ref{thm:nesterov_acc_pgd}]
This proof mirrors that in \citet{Bubeck2015} for the unconstrained case. Denote by $\alpha$ the step-size and $g^{\perp}(x_t)$ the orthogonal projection of $\nabla f(x_t)$ onto $\cC$
\[
g^{\perp}(x_t) = \frac{1}{\alpha}\left(x_t - \Pi_{\cC}(x_t - \alpha\nabla f(x_t)) \right)
\]
From Lemma \ref{lem:beta_smoothness},
\begin{align*}
f(z_{t+1}) - f(z_t) &\leq g^{\perp}(x_t)^T(x_t - z_t) - \frac{1}{2\beta}||g^{\perp}(x_t)||_2^2 \\
&= \beta(x_t - z_{t+1})^T(x_t-z_t) - \frac{\beta}{2}||x_t - z_{t+1}||_2^2, \numberthis \label{eqn:beta_smoothness_ineq1}
\end{align*}
where the equality follows by substituting in the update step for $z_{t+1}$. Similarly, we can find that
\begin{equation}
\label{eqn:beta_smoothness_ineq2}
f(z_{t+1}) - f(x^*) \leq \beta (x_t - z_{t+1})^T(x_t - x^*) - \frac{\beta}{2}||x_t - z_{t+1}||_2^2.
\end{equation}
Next, denote the distance between the value at the $t^{th}$ iterate and the optimal value by $\delta_t := f(z_t) - f(x^*)$. To bound $\delta_t$, we can multiply both sides of \eqref{eqn:beta_smoothness_ineq1} by $(\lambda_{t} - 1)$ and add \eqref{eqn:beta_smoothness_ineq2} to obtain the relation
\begin{equation}
\label{eqn:beta_smoothness_ineq3}
\lambda_t\delta_{t+1} - (\lambda_t -1)\delta_t \leq \beta(x_t - z_{t+1})^T(\lambda_t x_t - (\lambda_t -1)z_t - x^*) - \frac{\beta}{2}\lambda_t||x_t - z_{t+1}||_2^2.
\end{equation}
From the definition of $\lambda_t$ given in \eqref{eqn:aux_sequence_acc_pgd}, we can see that $\lambda_t^2-\lambda_t = \lambda_{t-1}^2$. Using this, we multiply \eqref{eqn:beta_smoothness_ineq3} by $\lambda_t$ on both sides, giving
\begin{align*}
\lambda_t^2\delta{t+1} - \lambda_{t-1}^2\delta_t &\leq \frac{\beta}{2}\left(2\lambda_t(x_t - z_{t+1})^T(\lambda_tx_t - (\lambda_t - 2)z_t - x^*) - ||\lambda_t(z_{t+1} - x_t)^2||_2^2 \right)\\
&= \frac{\beta}{2}\left(||\lambda_t x_t - (\lambda_t -1)z_t - x^*||_2^2 - ||\lambda_tz_{t+1} - (\lambda_t -1)z_t - x^*||_2^2 \right). \numberthis \label{eqn:beta_smoothness_ineq4}
\end{align*}
Now, if we multiply the update step for $x_t$ in Algorithm \ref{alg:nesterov_acc_pgd} by $\lambda_{t+1}$ on both sides we obtain the relation
\begin{equation}
\label{eqn:beta_smoothness_eq1}
\lambda_{t+1}x_{t+1} - (\lambda_{t+1} - 1)z_{t+1} = \lambda_t z_{t+1} - (\lambda_t -1)z_t.
\end{equation}
We can define $u_t = \lambda_t x_t - (\lambda_t - 1)z_t - x^*$ and substitute this into \eqref{eqn:beta_smoothness_ineq4} which gives
\[
\lambda_t^2 \delta_{t+1} - \lambda_{t-1}^2\delta_t^2 \leq \frac{\beta}{2}\left(||u_t||_2^2 - ||u_{t+1}||_2^2 \right).
\]
Summing these from $1$ to $T-1$, we see that they telescope, giving
\[
\delta_T \leq \frac{\beta}{2\lambda_{T-1}^2}||x_1 - x^*||.
\]
Lastly, for $T=2$, clearly $\lambda_{T-1} \geq \frac{T}{2}$. By an inductive argument, we easily obtain that for any $T$, $\lambda_{T-1}\geq \frac{T}{2}$. Plugging this in gives
\[
f(z_T) - f(x^*) \leq \frac{2\beta ||x_1 - x^*||^2}{T^2},
\]
as desired.
\end{proof}

\end{document}